\documentclass[12pt]{article}
\usepackage{times}

\usepackage{hyperref}       %
\usepackage{amsfonts}
\usepackage{amsmath}
\usepackage{amsthm}
\usepackage{bm}
\usepackage{float}
\usepackage{amssymb}
\usepackage{graphicx} 
\usepackage{nccmath}
\usepackage{thmtools}
\usepackage{thm-restate}
\usepackage{color}
\usepackage{cleveref}
\usepackage{algorithm}
\usepackage{algorithmic}
\usepackage{thm-restate}

\usepackage[thicklines]{cancel}

\newcommand{\X}{X}
\newcommand{\Y}{Y}
\newcommand{\Z}{Z}
\newcommand{\W}{W}
\newcommand{\Xs}{{\X_i}}
\newcommand{\Ys}{{\Y_i}}
\newcommand{\s}{{s}} %

\theoremstyle{plain}
\newtheorem{theorem}{Theorem}%
\newtheorem{proposition}[theorem]{Proposition}
\newtheorem{lemma}[theorem]{Lemma}
\newtheorem{corollary}[theorem]{Corollary}

\newtheorem{definition}[theorem]{Definition}
\newtheorem{assumption}[theorem]{Assumption}

\theoremstyle{definition}
\newtheorem{remark}[theorem]{Remark}

\newcommand{\eg}{\emph{e.g.}}

\newcommand{\cC}{\mathcal{C}}
\newcommand{\cD}{\mathcal{D}}

\newcommand{\cH}{\mathcal{H}}
\newcommand{\cI}{\mathcal{I}}

\newcommand{\cP}{\mathcal{P}}

\newcommand{\cX}{\mathcal{X}}
\newcommand{\cY}{\mathcal{Y}}

\newcommand{\bE}{\mathbb{E}}

\newcommand{\bM}{\mathbb{M}}

\newcommand{\bP}{\mathbb{P}}

\newcommand{\bR}{\mathbb{R}}

\usepackage{mathtools}
\usepackage{amssymb}
\usepackage{latexsym}
\usepackage{dsfont}
\usepackage{mathrsfs}
\usepackage{amssymb}
\usepackage{amsfonts}
\usepackage{bm}
\usepackage{xspace}

        {\medskip}

        {\hspace*{\fill}$\Box$\par\vspace{4mm}}
        {\hspace*{\fill}$\Box$\par}

\newcommand{\grad}{{\nabla}}

\ifx\BlackBox\undefined
\newcommand{\BlackBox}{\rule{1.5ex}{1.5ex}}  %
\fi

\ifx\QED\undefined
\def\QED{~\rule[-1pt]{5pt}{5pt}\par\medskip}
\fi

\ifx\theorem\undefined
\newtheorem{theorem}{Theorem}[section]
\fi

\ifx\example\undefined
\newtheorem{example}{Example}[section]
\fi

\ifx\property\undefined
\newtheorem{property}{Property}
\fi

\ifx\lemma\undefined
\newtheorem{lemma}{Lemma}[section]
\fi

\ifx\proposition\undefined
\newtheorem{proposition}{Proposition}[section]
\fi

\ifx\remark\undefined
\newtheorem{remark}{Remark}
\fi

\ifx\corollary\undefined
\newtheorem{corollary}{Corollary}[section]
\fi

\ifx\definition\undefined
\newtheorem{definition}{Definition}[section]
\fi

\ifx\conjecture\undefined
\newtheorem{conjecture}{Conjecture}[section]
\fi

\ifx\axiom\undefined
\newtheorem{axiom}[theorem]{Axiom}
\fi

\ifx\claim\undefined
\newtheorem{claim}[theorem]{Claim}
\fi

\ifx\assumption\undefined
\newtheorem{assumption}{Assumption}
\fi

\ifx\problem\undefined
\newtheorem{problem}{Problem}[section]
\fi

\ifx\fact\undefined
\newtheorem{fact}{Fact}
\fi

\usepackage{hyperref}
\usepackage{url}

\usepackage{fancyhdr}
\usepackage{multirow}
\usepackage{booktabs}
\usepackage{amsfonts}
\usepackage{nicefrac}
\usepackage{microtype}
\usepackage{xcolor}
\usepackage[round]{natbib}
\usepackage{enumitem}
\usepackage{graphicx}
\usepackage{subfigure}
\usepackage{algorithm}
\usepackage{algorithmic}
\usepackage{amsmath,amsfonts,bm}
\usepackage{amsthm,amssymb}
\usepackage{mathtools}
\usepackage{bbm}
\usepackage{diagbox}
\usepackage{color}
\usepackage{float}
\usepackage{lipsum}
\usepackage{multirow}
\usepackage{graphicx}
\usepackage{caption}
\usepackage{wrapfig}
\usepackage{selectp}
\renewcommand{\bibname}{References}
\renewcommand{\bibsection}{\subsubsection*{\bibname}}
\usepackage{eso-pic}
\usepackage{forloop}
\date{}

\usepackage{parskip}
\title{\Large{Predictive Inference With Fast Feature Conformal Prediction}}
\author{
  Zihao Tang\\
  \small{School of Statistics and Management} \\ 
  \small{Shanghai University of Finance and Economics} \\
\small{\texttt{tangzihao@stu.sufe.edu.cn}} 
  \and
  Boyuan Wang\\
  \small{School of Economics and Management} \\ 
  \small{Tiangong University}\\
\small{\texttt{2110610065@tiangong.edu.cn}}
  \and
  Chuan Wen\\
 \small{Institute for Interdisciplinary Information Sciences} \\ 
 \small{Tsinghua University} \\
\small{\texttt{cwen20@mails.tsinghua.edu.cn}}
  \and
  Jiaye Teng\thanks{Corresponding author.}\\
  \small{School of Statistics and Management} \\ 
  \small{Shanghai University of Finance and Economics} \\
  \small{\texttt{tengjiaye@sufe.edu.cn}}
}

\ifdefined\usebigfont

\usepackage{times}
\usepackage[fontsize=13pt]{scrextend}
\usepackage[left=1.56in,right=1.56in,top=1.74in,bottom=1.74in]{geometry}
\else
\usepackage[margin=1.28in]{geometry}
\fi

\begin{document}

\maketitle

\begin{abstract}

Conformal prediction is widely adopted in uncertainty quantification, due to its post-hoc, distribution-free, and model-agnostic properties.
In the realm of modern deep learning, researchers have proposed Feature Conformal Prediction (FCP), which deploys conformal prediction in a feature space, yielding reduced band lengths.
However, the practical utility of FCP is limited due to the time-consuming non-linear operations required to transform confidence bands from feature space to output space.
In this paper, we introduce Fast Feature Conformal Prediction (FFCP), which features a novel non-conformity score and is convenient for practical applications.
FFCP serves as a fast version of FCP, in that it equivalently employs a Taylor expansion to approximate the aforementioned non-linear operations in FCP.
Empirical validations showcase that FFCP performs comparably with FCP (both outperforming the vanilla version) while achieving a significant reduction in computational time by approximately 50x. The code is available at \url{https://github.com/ElvisWang1111/FastFeatureCP}

\end{abstract}

\section{Introduction}
\label{sec: intro}
Machine learning has been successfully applied into various fields~\citep{jordan2015machine,silver2017mastering}. 
However, machine learning models often face overconfidence issues~\citep{wei2022mitigating} and even hallucinations in large language models (LLMs)~\citep{ji2023survey}, which makes them unreliable and cannot be deployed in fields like finance and medicines~\citep{gelijns2001uncertainty,thirumurthy2019uncertain,morduch2017financial}. 
Therefore, it is essential to develop techniques for uncertainty quantification and calibrate the original model~\cite{abdar2021review,Guo2017OnCO,chen2021neural, DBLP:journals/corr/abs-2107-03342}
Among the uncertainty quantification techniques, Conformal Prediction (Vanilla CP or split conformal prediction, \cite{vovk2005algorithmic}; \cite{shafer2008tutorial}; \cite{burnaev2014efficiency}) stands out, because it is distribution-free, does not require retraining, and can be directly applied to various models. 
Conformal prediction uses a calibration step to calibrate a base model and then construct the confidence band. 
The goal of conformal prediction is to return a band $\cC_{1-\alpha}$ such that 
\begin{equation}
    \bP(\Y^\prime \in \cC_{1-\alpha}(\X^\prime)) \geq 1-\alpha.
\end{equation}
where $(X^\prime, Y^\prime)$ is a test point and $1-\alpha$ represents the confidence level.

In deep learning regimes, researchers try to utilize feature information in Vanilla CP, since the feature space usually contains meaningful semantic information in neural networks~\citep{shen2014learning}. 
This leads to Feature Conformal Prediction (FCP, \cite{teng2022predictive}).
Fortunately, one may get different band lengths on different individuals by using the feature information, leading to a shorter confidence band. 
For comparison, on the regression task, Vanilla CP only returns the same band length for all individuals, which indicates a longer length.
We refer to \citet{teng2022predictive} for more discussion.

\begin{figure}[t]
    \centering
    \subfigure{\includegraphics[height=0.6\linewidth]{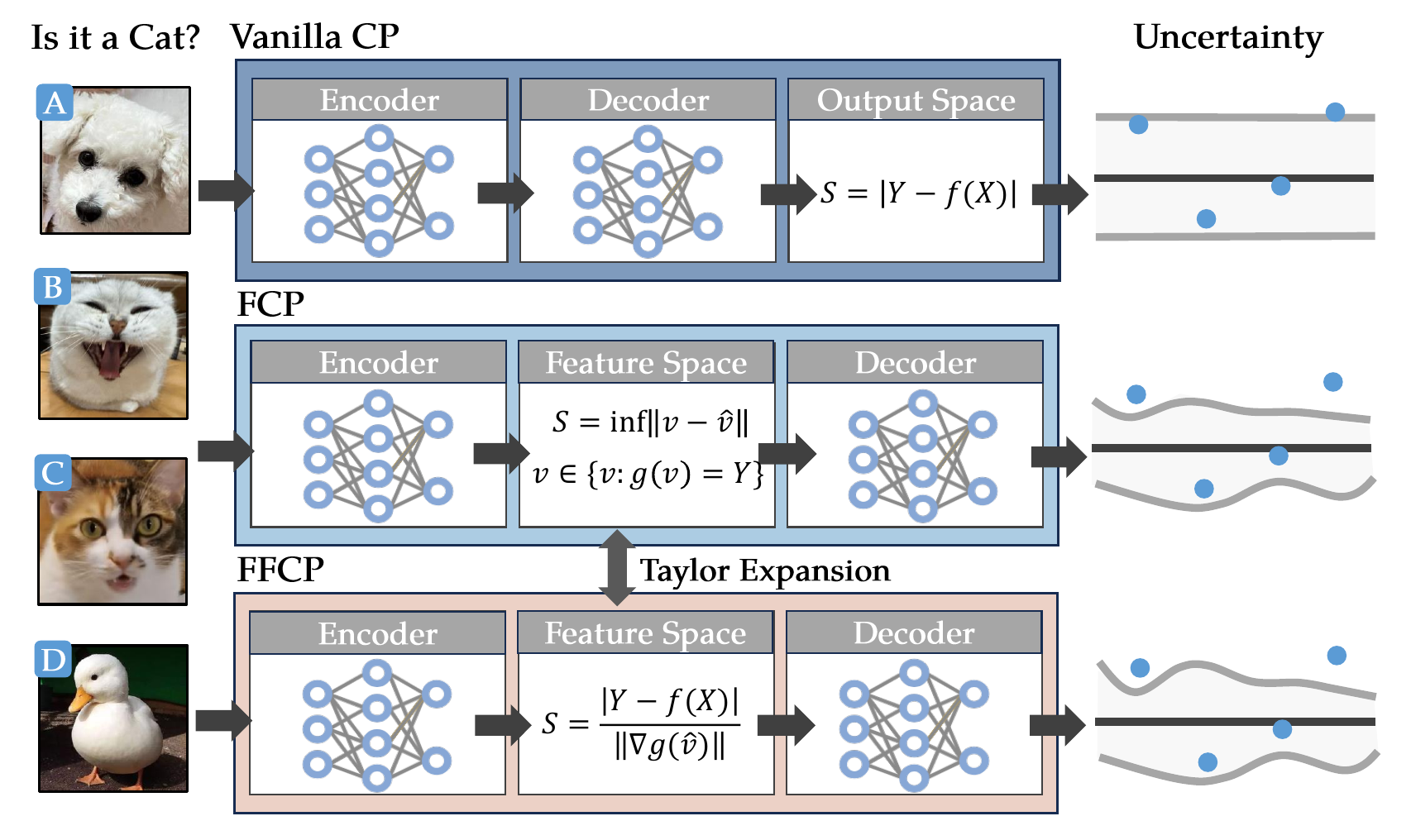}}
    \vspace{-0.2cm}
    \caption{Comparison among Vanilla CP, FCP, and FFCP. 
    FCP and FFCP are more efficient compared to Vanilla CP since they return different band lengths for different individuals. 
    This is done by calculating a non-conformity score in the feature space.
    Besides, FFCP approximates FCP using a Taylor expansion, which leads to a different non-conformity score and accelerates the transformation from feature space to output space. }
    \label{fig:inuitive-illustration}
\end{figure}

However, the practical applications of FCP are limited because (a) it is time-consuming, and (b) it only returns \emph{estimated} bands on the output space, making it less efficient. 
These two issues come from the step \emph{Band Estimation}, which transfers the confidence band from feature space to output space. 
This step involves complex non-linear operations called LiPRA~\citep{xu2020automatic} and therefore (a) the non-linear operation requires high computational costs, and (b) the configurations in LiPRA might finally influence the estimated band, further harming the performance of FCP.

In this paper, we present Fast Feature Conformal Prediction (FFCP), which offers a fast version for handling the aforementioned nonlinear operations. 
Different from Vanilla CP and FCP, FFCP introduces a novel non-conformity score $\s_{\text{ff}}(\cdot)$ that is simple to compute and does not require additional training,
\begin{equation}
    \s_{\text{ff}}(\X, \Y, {g} \circ {h}) = |\Y -  {g} \circ {h}(\X)|/\| \grad {g}(\hat{v})\|,
\end{equation}
where $(X, Y)$ denotes a sample, $g \circ h$ denotes a neural network where $h$ denotes the feature layers and $g$ denotes the prediction head, and the gradient $\grad {g}(\hat{v})$ denotes the gradient of ${g}(\cdot)$ on the feature $\hat{v}_i \triangleq h(\X)$, namely, $\grad {g}(\hat{v}_i) = \frac{\mathrm{d} {g}\circ {h}(\X_i)}{\mathrm{d} {h}(\X_i)} $.
We refer to Algorithm~\ref{alg: ffcp} for more details and illustrate the algorithm in Figure~\ref{fig:inuitive-illustration}.

The above non-conformity score is closely related to FCP.
Specifically, \textbf{FFCP using this non-conformity score can be regarded as a fast version of FCP, since it equivalently approximates the prediction head using a Taylor expansion}, which simplifies the aforementioned nonlinear operations.
Fortunately, FFCP inherits the merits of FCP, for example, it also utilizes the semantic information in the feature. 
We refer to Section~\ref{subsec:more-discussion} for more discussion.

In the theoretical perspective, we first demonstrate that FFCP is effective in Theorem~\ref{thm: ffcp effective}, in that it returned a confidence band with empirical coverage larger than the given confidence $1-\alpha$. 
Additionally, we demonstrate in Theorem~\ref{thm:FFCP-efficient-abs} that under square conditions, FFCP produces a shorter confidence band than Vanilla CP.
The square conditions outline the properties of the feature space from two perspectives: expansion and quantile stability, implying that the feature space has a smaller distance between individual non-conformity scores and their quantiles. 
This reduces the cost of the quantile operation and therefore leads to a shorter confidence band. 
We validate the square conditions using empirical observations. 

Empirically, we conduct several experiments on real-world datasets and show that FFCP performs comparably with FCP, both outperforming Vanilla CP, \textbf{while achieving nearly 50 times the speed of FCP in terms of runtime} for regression tasks.
We further validate the approximation ability of FFCP with FCP using the correlation between the non-conformity score of FFCP and FCP. 
We also apply FFCP to the image segmentation problems to verify its general applications. 
Besides, we show that the concept in FFCP is pretty general, and can be combined with other variants of CP, \eg, CQR~\citep{romano2019conformalized} and LCP~\citep{guan2023localized} in regression tasks, and RAPS~\citep{angelopoulos2020uncertainty} in classification tasks.

Overall, our main contributions are summarized as follows:
\begin{itemize}
\item We propose FFCP, which serves as a fast version of FCP. 
FFCP approximately achieves a 50x reduction in running time compared to FCP (Table~\ref{tab: FFCP time}) by using a Taylor expansion to approximate the prediction head in FCP.
Besides, FFCP inherits the merits of FCP and efficiently exploits semantic information in the feature space like FCP. 

\item Theoretical insights demonstrate that FFCP returns shorter band length compared to Vanilla CP (Theorem~\ref{thm:FFCP-efficient-abs}) while ensuring coverage exceeds the given confidence level under some mild conditions (Theorem~\ref{thm: ffcp effective}).

\item Extensive experiments with both synthetic and real data demonstrate the effectiveness of the proposed FFCP algorithm (Table~\ref{tab1:1dim-results}). Additionally, we demonstrate the universal applicability of our gradient-level techniques by extending them to other tasks such as classification (FFRAPS, Algorithm~\ref{alg:ffraps}) and segmentation, and to various conformal prediction variants, including CQR (FFCQR, Algorithm~\ref{alg: ffcqr}) and LCP (FFLCP, Algorithm~\ref{alg: fflcp}).

\end{itemize}

\section{Related Work}
Conformal prediction is a post-hoc calibration method dealing with uncertainty quantification~\citep{vovk2005algorithmic,DBLP:journals/jmlr/ShaferV08,Barber2020TheLO}, which is deployed in numerous fields~\citep{ye2024benchmarking,kumar2023conformal,quach2023conformal}.
The variants of conformal prediction typically revolve around the concept of non-conformity scores, with four main branches of development.

\paragraph{Relaxing Exchangeability.}
The first branch relaxes the exchangeability requirement~\citep{DBLP:conf/nips/TibshiraniBCR19, Hu2020ADT, Podkopaev2021DistributionfreeUQ, barber2022conformal},  leveraging weighted or reweighted quantiles to relax exchangeability.
By doing so, it gains more flexibility and broader applicability in handling data that may not satisfy the standard exchangeability assumptions.

\paragraph{Diverse Structures.}
The second branch applies conformal prediction to various data structures, for example, classification tasks~\citep{Romano2020ClassificationWV,angelopoulos2020uncertainty}, time series data~\citep{Xu2021ConformalPI,Gibbs2021AdaptiveCI}, censored data in survival analysis~\citep{DBLP:conf/icml/TengTY21,candes2023conformalized}, high-dimensional data~\citep{candes2021conformalized,Lei2013ACP}, Bellman-based data~\citep{yang2024bellman}, 
counterfactuals and individual treatment effects\citep{Lei2021ConformalIO}, 
and so on.

Another way involves model structures, such as $k$-NN regression~\citep{DBLP:journals/jair/PapadopoulosVG11}, quantiles incorporated~\citep{DBLP:conf/nips/RomanoPC19,sesia2020comparison}, density estimators~\citep{Izbicki2020DistributionfreeCP}, and conditional histograms~\citep{DBLP:conf/nips/SesiaR21}.
These methods further enrich the application scenarios of conformal prediction by adapting it to diverse model frameworks.

\paragraph{Enhancing Methods.}
The third branch focuses on enhancing the original conformal prediction with band length. \citet{izbicki2020cd} introduce CD-split and HPD-split methods, and \citet{yang2021finite} develop selection methods to minimize band length.
Of particular note is feature conformal prediction~\citep{teng2022predictive}, which leverages neural network training information via feature spaces to improve band length.

\paragraph{Localized Conformal Prediction.}  
The fourth branch focuses on enhancing the non-conformity score normalization by incorporating difficulty-related terms like $\frac{\|Y - \hat{Y}\|}{\sigma(X^\prime)}$ where $Y$ denotes the true label, $\hat{Y}$ denotes the predicted label, and 
$\sigma(X')$ denotes the standard deviation related to $X^\prime$. Here are three key approaches:
\begin{enumerate}
    \item \textbf{Weight Adjustment via Calibration Distances.}  
    This approach mainly calculates the distance from the testing point  to the calibration points and then uses these distances to define the weights of non-conformity scores in the calibration process~\citep{han2022split,guan2023localized}. Our gradient-level techniques can be used to combine with this branch (see FFLCP proposed later).
    
    \item \textbf{Normalization Using Proximity to Training Set.}  
    This approach utilizes the observation that a testing point exhibits smaller uncertainty when it is close to the training set, and uses such metrics to approximate $\sigma(X')$~\citep{papadopoulos2008normalized,papadopoulos2011regression,papadopoulos2011reliable}.
    In the deep learning regimes, we believe that such procedures can be further improved by calculating the distance in the feature space rather than the input space, since feature layers usually contain more semantic information.
    
    \item \textbf{Modeling $\sigma(X')$.}  
    A model is trained to estimate $\sigma(X')$, offering computational efficiency and reducing the need for additional training data~\citep{seedat2023improving,seedat2024triage}.
\end{enumerate}

\textbf{Uncertainty Quantification.} Uncertainty quantification is one of the most fundamental questions in machine learning.
In addition to conformal prediction (CP) methods, many other approaches exist for quantifying uncertainty. 
Common techniques include calibration-based techniques~\citep{Guo2017OnCO,DBLP:conf/icml/KuleshovFE18,DBLP:conf/cvpr/NixonDZJT19,abdar2021review,chang2024survey} and Bayesian-based techniques~\citep{DBLP:conf/icml/BlundellCKW15,DBLP:conf/icml/Hernandez-Lobato15b,DBLP:conf/icml/LiG17,izmailov2021bayesian,jospin2022hands}. 

\section{Preliminaries}
\label{sec: prelim}
We begin by introducing a dataset $\mathcal{D} = \{(X_i, Y_i)\}_{i \in [n]}$ indexed by $\cI$.
We split the dataset into two folds: a training fold $\mathcal{D}_{\text{tra}}$ indexed by $\mathcal{I}_{\text{tra}}$, and a calibration fold $\mathcal{D}_{\text{cal}}$ indexed by $\mathcal{I}_{\text{cal}}$.
Denote the testing point by $(\X^\prime, \Y^\prime)$.
For the model part, define $f$ as a neural network.
We partition $f = g \circ h$, where $h$ denotes the feature function (the initial layers of the neural network) and $g$ denotes the prediction head. 
For a sample $(\X, \Y)$, we define $\hat{v} = {h}(\X)$ as the trained feature.
We follow the ideas in \citet{teng2022predictive} and define the surrogate feature as any feature $v$ such that ${g}(v) = \Y$.
\begin{assumption}[exchangeability]
\label{assu: exchangeability}
Assume that the calibration data $(\Xs, \Ys) \in \mathcal{D}_{\text{cal}}$ and the testing point $(\X^\prime, \Y^\prime)$ are exchangeable. Formally, define $\Z_i, i=1, \dots, |\cI_{\text{cal}}| + 1$, as the above data pair. Then $\Z_i$ are exchangeable if arbitrary permutation follows the same distribution, i.e.,
\begin{equation}
(\Z_{1}, \dots, \Z_{|\cI_{\text{cal}}| + 1} ) \overset{d}{=} (\Z_{\pi(1)}, \dots, \Z_{\pi(|\cI_{\text{cal}}| + 1)}),
\end{equation}
with arbitrary permutation $\pi$ over $\{1, \cdots, |\cI_{\text{cal}}| + 1\}$.
\end{assumption}

Typically, Vanilla CP is composed of three key steps.\\
\textbf{I. Training Step.} 
We first train a base model using the training fold $\mathcal{D}_{\text{tra}}$.\\
\textbf{II. Calibration Step.} 
We calculate a non-conformity score $R_i=|\Y_i - f(\X_i)|$  using the calibration fold $\mathcal{D}_{\text{cal}}$.
The form of the score function might vary case by case, quantifying the divergence between ground truth and predicted values. \\
\textbf{III. Testing Step.} 
We construct the confidence band for the testing point $(X^\prime, Y^\prime)$ using the quantile of the non-conformity score $Q_{1-\alpha}$.

We present Vanilla CP in Algorithm~\ref{alg: cp}~\footnote{$\delta$ represents the Dirac function.}, and provide its theoretical guarantee in Theorem~\ref{thm:cp}.
 
\begin{algorithm*}[t]
\caption{Vanilla Conformal Prediction} 
\label{alg: cp} 
\begin{algorithmic}[1] 
\REQUIRE Confidence level $\alpha$, dataset $\cD = \{ (\Xs, \Ys)\}_{i \in \cI}$, tesing point $X^\prime$

\STATE{Randomly split the dataset $\mathcal{D}$ into a training fold $\cD_{\text{tra}} \triangleq \{(\X_i, \Y_i)\}_{i \in \mathcal{I}_{\text{tra}}}$ and a calibration fold $\cD_{\text{cal}} \triangleq \{(\X_i, \Y_i)\}_{i \in \mathcal{I}_{\text{cal}}}$; }

\STATE{Train a base model $f(\cdot) $ with training fold $\cD_{\text{tra}}$ ;}

\STATE{For each $i \in \mathcal{I}_{\text{cal}}$, calculate the non-conformity score $R_i = |\Y_i - f(\X_i)|$;}

\STATE{Calculate the $(1-\alpha)$-th quantile $Q_{1-\alpha}$ of the distribution  $\frac{1}{|\cI_{\text{cal}}| + 1} \sum_{i \in \mathcal{I}_{\text{cal}}  \delta_{R_i} + {\delta}_{\infty}}$.}

\ENSURE ${\cC^{\text{Vanillacp}}_{1-\alpha}(\X^\prime) = [f(\X^\prime) - Q_{1-\alpha}, f(\X^\prime) + Q_{1-\alpha}]}$.
\end{algorithmic} 
\end{algorithm*}

\begin{theorem}
\label{thm:cp}
Under Assumption~\ref{assu: exchangeability}, the confidence band $\cC_{1-\alpha}(\X^\prime)$ returned by Algorithm~\ref{alg: cp} satisfies
\begin{equation}
    \bP(\Y^\prime \in \cC_{1-\alpha}(\X^\prime)) \geq 1-\alpha.
\end{equation}
\end{theorem}

\section{Methodology}
\label{sec: ffcp}
In this section, we first illustrate the motivation behind FFCP in Section~\ref{subsec: Relationship between FFCP and FCP}.
Specifically, we address the complexity of non-linear operators in FCP and provide how we derive FFCP from FCP.
We then formally present the specific form of FFCP, including the non-conformity score, the returned bands, and the corresponding pseudocode.
We finally provide theoretical analyses on the coverage and band length in Section~\ref{subsec:Theoretical Guarantee for FFCP}.

\subsection{Relationship between FFCP and FCP}
\label{subsec: Relationship between FFCP and FCP}

In this section, we discuss the motivation behind FFCP. 
FFCP is inspired by FCP~\citep{teng2022predictive}, which conducts conformal prediction in the feature space. 
However, since the band is constructed in the feature space, FCP requires a \emph{Band Estimation} process to go from feature space to output space. 
Specifically, FCP applies \emph{LiPRA}~\citep{xu2020automatic} which derives the band in the output space $\{g(v): \| v - \hat{v} \| \leq Q_{1-\alpha}\}$. 
Unfortunately, the exact band is difficult to represent explicitly since the prediction head $g$ is usually highly non-linear, thereby resulting in significant computational complexity in terms of time.
Therefore, we propose approximating $g$ using a first-order Taylor expansion to simplify the aforementioned non-linear operator.

The core steps of FCP include: (a) calculating the non-conformity score (from output space to feature space),  followed by (b) deriving the confidence band (from feature space to output space).
We next introduce the concrete formulation of how FFCP approximates FCP. 

\textbf{From output space to feature space. }
FCP uses the non-conformity score $s_{\text{f}}(\cdot)$ in the feature space
\begin{equation}
\label{eqn: score}
    \s_{\text{f}}(\X, \Y, g \circ h) = \inf_{v \in \{v: g(v) = \Y\}}\|v -  \hat{v}\|.
\end{equation}
By using the Taylor expansion, one approximates $g$ with $g(v) \approx g(\hat{v}) + \grad g(\hat{v})(v-\hat{v})$.
Plugging into the approximation of $g$ leads to a new non-conformity score $s_{\text{ff}}(\cdot)$ 
\begin{equation}
    \s_{\text{ff}}(\X, \Y, g \circ h) = |\Y -  f(\X)|/\| \grad g(\hat{v})\|,
\end{equation}
where $\grad g(\hat{v})$ denotes the gradient of $g(\hat{v})$ on the feature $\hat{v}$, namely $\grad g(\hat{v}) = \frac{\mathrm{d} g\circ h(\X)}{\mathrm{d} h(\X)} $.

\textbf{From feature space to output space. }
After constructing the confidence band in the feature space, FCP maps this band to the output space. 
Specifically, they derive the following band in the output space which is called \emph{Band Estimation}:
\begin{equation}
    \{g(v): \| v - \hat{v} \| \leq Q_{1-\alpha}\}.
\end{equation}
They propose to use LiPRA in this process, which is time-consuming. By plugging into the Taylor approximation of $g$, one can construct the band $\cC^{\text{ffcp}}_{1-\alpha}$ as
\begin{equation}
\cC^{\text{ffcp}}_{1-\alpha}(\X)=\left[g(\hat{v}) - \|\nabla g(\hat{v})\| Q_{1-\alpha},  g(\hat{v}) + \|\nabla g(\hat{v})\| Q_{1-\alpha}\right].
\end{equation}

\begin{remark}[High-dimensional Response]
    When the response $\Y= [\Y^1, \Y^2, \ldots, \Y^m]$ is high-dimensional, one can deploy conformal prediction in a coordinate-wise level. In this scenario, the confidence band for a specific dimension $j \in [m]$ of $\Y$ becomes
    \begin{equation}
        \s^j_{\text{ff}}(\X, \Y, g \circ h) = |\Y^j -  f(\X)^j|/\| \grad g(\hat{v})^j\|,
    \end{equation}
    where $\nabla g(\hat{v})^j = \left(\frac{\partial f(\X)}{\partial h(\X)}\right)_j$ represents the $j$-th row of the Jacobian matrix of $f$ with respect to $h$ at $\X$.
    And the returned band for the $j$-th coordinate is derived as follows:
    \begin{equation}
    \cC^{\text{ffcp}}_{1-\alpha}(\X)_j = \left[g(\hat{v})^j - \|\nabla g(\hat{v})^j\| Q_{1-\alpha}, g(\hat{v})^j + \|\nabla g(\hat{v})^j\| Q_{1-\alpha}\right].
    \end{equation}
\end{remark}

Based on the above discussion, we present the full algorithm in Algorithm~\ref{alg: ffcp}.
Notably, the Taylor expansion in FFCP is usually different for each sample $\X, \Y$, which further leads to confidence bands that are individually different. 
Besides, FFCP inherits the advantages of FCP.
For example, this framework is pretty general and can be combined with other variants of Vanilla CP, \eg, CQR~\citep{romano2019conformalized}.

\begin{algorithm*}[t]
\caption{Fast Feature Conformal Prediction 
} 
\label{alg: ffcp} 
\begin{algorithmic}[1] 
\REQUIRE  Confidence level $\alpha$, dataset $\cD = \{ (\Xs, \Ys)\}_{i \in \cI}$, tesing point $X^\prime$
\STATE{Randomly split the dataset $\mathcal{D}$ into a training fold $\cD_{\text{tra}} \triangleq \{(\X_i, \Y_i)\}_{i \in \mathcal{I}_{\text{tra}}}$ and a calibration fold $\cD_{\text{cal}} \triangleq \{(\X_i, \Y_i)\}_{i \in \mathcal{I}_{\text{cal}}}$ ;}

\STATE{Train a base neural network with training fold $f(\cdot) = g\circ h(\cdot)$ with training fold $\cD_{\text{tra}}$;}

\STATE{For each $i \in \mathcal{I}_{\text{cal}}$, calculate the non-conformity score 
$\tilde{R}_i = |\Y_i -  f(\X_i)|/\| \grad {g}(\hat{v}_i)\|,$ where $\grad {g}(\hat{v}_i)$ denotes the gradient of ${g}(\cdot)$ on the feature $\hat{v}_i \triangleq h(\X_i)$, namely $\grad {g}(\hat{v}_i) = \frac{\mathrm{d} {g}\circ {h}(\X_i)}{\mathrm{d} {h}(\X_i)} $;
}

\STATE{Calculate the $(1-\alpha)$-th quantile $Q_{1-\alpha}$ of the distribution  $\frac{1}{|\cI_{\text{cal}}| + 1} \sum_{i \in \mathcal{I}_{\text{cal}}}  \delta_{\tilde{R}_i} + {\delta}_{\infty}$;}

\ENSURE $\cC_{1-\alpha}^{\text{ffcp}}(\X^\prime)= \left[ f(\X^\prime) - \|\nabla{g}(\hat{v}^\prime)\| Q_{1-\alpha}, f(\X^\prime) + \|\nabla {g}(\hat{v}^\prime)\| Q_{1-\alpha}\right]$, where $\hat{v}^\prime = h(\X^\prime)$.

\end{algorithmic} 
\end{algorithm*}

\subsection{Theoretical Guarantee for FFCP}
\label{subsec:Theoretical Guarantee for FFCP}
This section outlines the theoretical guarantee for FFCP concerning coverage (effectiveness) and band length (efficiency). 
Below, we offer the main theorems and defer the full proofs to Appendix~\ref{appendix: ffcp_effective} and \ref{appendix: ffcp_efficient}.

We first demonstrate that the confidence band produced by Algorithm~\ref{alg: ffcp} is valid under Assumption~\ref{assu: exchangeability}.

\begin{theorem}[Coverage]
\label{thm: ffcp effective}
Under Assumption 1, for any $\alpha > 0$, the confidence band returned by Algorithm~\ref{alg: ffcp} satisfies:
\begin{equation}
    \mathbb{P}(Y' \in \cC^{\text{ffcp}}_{1-\alpha}(X')) \geq 1 - \alpha,
\end{equation}
where the probability is taken over the calibration fold and the testing point $(X', Y')$.
\end{theorem}

Next, we show that FFCP is provably more efficient than the Vanilla CP.
To simplify the discussion, we present an informal version of Theorem~\ref{thm:FFCP-efficient-abs} here and postpone the formal version to Appendix~\ref{appendix: ffcp_efficient}.

\begin{theorem}[Band Length]
\label{thm:FFCP-efficient-abs}
Under mild assumptions, if the
following square conditions hold:

\vspace{-0.05cm}
\begin{enumerate}
    \item \textbf{Expansion.} The feature space expands the differences between individual length and their quantiles.
    
    \item \textbf{Quantile Stability.} Given a calibration set $\cD_{\text{cal}}$, the quantile of the band length is stable in both feature space and output space. 
    \vspace{-0.05cm}
\end{enumerate}
Then FFCP provably outperforms vanilla CP in terms of average band length.
\end{theorem}

Theorem~\ref{thm:FFCP-efficient-abs} implies that FFCP provably outperforms Vanilla CP in terms of average band length. 
The \emph{square conditions} imply that the feature space has a smaller distance between individual non-conformity scores and their quantiles. 
This reduction in the computational overhead of the quantile operation subsequently yields a shorter band length.
We provide empirical verifications on this assumption, see Figure~\ref{fig:VCP-FFCP-score-Histogram} for more details.

The intuition behind Theorem~\ref{thm:FFCP-efficient-abs} is as follows: Initially, FFCP and Vanilla CP perform quantile operations in different spaces, with the \emph{Expansion} condition ensuring that the quantile step in FFCP costs less. The ultimate \emph{Quantile Stability} condition confirms that the band can be generalized from the calibration folds to the test folds. 

\section{Experiments}
\label{sec: experiments}
This section presents the experiments to validate the utility of FFCP.
Firstly, we detail the experimental setup in Section~\ref{sec: exp setup}.
Secondly, we present that FFCP achieves both effectiveness and efficiency with faster execution in Section~\ref{subsec:FFCP-res}.
Thirdly, in Section~\ref{subsubsec:Other-task}, we verify that FFCP can be easily deployed and performs robustly across various tasks, including classification and segmentation.
Furthermore, in Section~\ref{subsubsec: Extensions-ffcp}, we show that the gradient-level techniques used in FFCP can be extended to classic CP models such as CQR~\citep{romano2019conformalized} and LCP~\citep{guan2023localized}. 
A more detailed account of this extension can be found in Section~\ref{subsec:Extensions}.
Finally, we conduct several more experiments in Section~\ref{subsec:more-discussion} to establish the close relationship between FFCP and FCP, to demonstrate the benefits of FFCP from the good representation of gradient, and to provide empirical validations for the theoretical insights.

\subsection{Experiments Setups}
\label{sec: exp setup}
\textbf{Datasets.} 
We consider both synthetic datasets and realistic datasets, including
\textbf{(a) synthetic dataset}: $\Y=\W\X+\epsilon$, where $\X \in [0,1]^{100}, \Y \in \mathbb{R}, \epsilon \sim \mathcal{N}(0,1)$, $\W$ is a fixed randomly matrix.
\textbf{(b)  real-world unidimensional target datasets}: 
ten datasets from UCI machine learning~\citep{Asuncion2007UCIML} and other sources: community and crimes (\emph{COM}), Facebook comment volume variants one and two (\emph{FB1} and \emph{FB2}), medical expenditure panel survey (\emph{MEPS19--21})~\citep{Cohen2009TheME}, Tennessee’s student teacher achievement ratio (\emph{STAR})~\citep{achilles2008tennessee},
physicochemical properties of protein tertiary structure (\emph{BIO}), blog feedback (\emph{BLOG})~\citep{buza2014feedback},
and bike sharing (\emph{BIKE}).
and 
\textbf{(c) real-world semantic segmentation dataset}: Cityscapes~\citep{Cordts2016Cityscapes}. \textbf{(d) real-world semantic classification dataset}: Imagenet-Val.

\textbf{Algorithms.} 
We compare three methods: Vanilla CP, FCP, and FFCP, with Vanilla CP serving as the baseline. For the one-dimensional scenario, we perform direct calculations. For higher-dimensional cases, we use a coordinate-wise level non-conformity score.

\textbf{Evaluation.}
The algorithmic empirical performance is evaluated with the following metrics:
\begin{itemize}
    \item \textbf{Runtime} For runtime evaluation of the algorithm, the timing starts at the score calculation and ends with the final prediction bands returned. All the tests are performed on a desktop with an Intel Core i9-12900H CPU, NVIDIA GeForce RTX 4090 GPU, and 32 GB memory.
    \item \textbf{Coverage (Effectiveness)} Coverage refers to the observed frequency with which a test point falls within the predicted confidence interval. Ideally, a predictive inference method should yield a coverage rate that is slightly higher than $1-\alpha$ for a given confidence level $\alpha$.
    \item \textbf{Band length (Efficiency)} When the coverage exceeds $1-\alpha$, our goal is to minimize the length of the confidence band. For FFCP, since we use a 5-layer neural network, each layer can be viewed as a feature layer. Therefore, in the experiments, we obtain the band length returned by each of the 5 layers of the neural network. In the subsequent results, if only a single band length is presented, it corresponds to the shortest band length returned by the different neural network layers. Otherwise, the results for all layers from layer 0 to layer 4 (with the last layer typically representing the Vanilla CP result) will be shown.
\end{itemize}
Let $\Y = (\Y^{1}, \ldots, \Y^{d}) \in \bR^d$ denote the $d$-dimensional response variable, and let $\cC(X) \subseteq \bR^d$ be the confidence band associated with predictor $X$. 
The length of this confidence band in each dimension is represented by the vector $|\cC(X)| = \left(|\cC(X)|^{1}, \ldots, |\cC(X)|^{d}\right) \in \bR^d$. 
Denote the indices of the test set by $\cI_{\text{tes}}$ and the set of dimensions by $[d] = \{1, \ldots, d\}$.

We then define the coverage and band length as: \begin{equation} \text{Coverage} = \frac{1}{|\cI_{\text{tes}}|} \sum_{i \in \cI_{\text{tes}}} \mathbb{I}\left(Y_i \in \cC(X_i)\right), \quad \text{Band Length} = \frac{1}{|\cI_{\text{tes}}|} \sum_{i \in \cI_{\text{tes}}} \left( \frac{1}{d} \sum_{j=1}^{d} |\cC(X_i)|^{j} \right), \end{equation} where $\mathbb{I}(\cdot)$ is the indicator function that equals 1 if its argument is true and 0 otherwise.

\subsection{Results on Coverage, Band Length and Runtime}
\label{subsec:FFCP-res}

\begin{table*}[t]
\caption{{Time comparison among Vanilla CP, FCP, and FFCP. FFCP ensures faster running speed compared to FCP. The last column represents the speed improvement factor of FFCP compared to FCP. The time unit is in seconds.}
}
\label{tab: FFCP time}
\begin{center}
\begin{small}
\begin{sc}
\begin{tabular}{ccccc}
\toprule
Dataset  &  \multicolumn{1}{c}{Vanilla CP}  &  \multicolumn{1}{c}{FCP}
&  \multicolumn{1}{c}{FFCP}
&  \multicolumn{1}{c}{FASTER}\\
\midrule
\midrule
synthetic  &  $0.0088$\scriptsize{$\pm 0.0003$} & ${3.8939}$\scriptsize{$\pm 0.3725$} &  $0.0902$\scriptsize{$\pm 0.0056$} & 
$43$x\\
com  &  $0.0047$\scriptsize{$\pm 0.0010$} & ${4.9804}$\scriptsize{$\pm 0.8588$}  & $0.0844$\scriptsize{$\pm 0.0187$} & 
$59$x\\
fb1 &  $0.0245$\scriptsize{$\pm 0.0059$}&  $5.9822$\scriptsize{$\pm 0.9871$}&   {$ 0.1940$\scriptsize{$\pm 0.0564$}} & 
$31$x\\
fb2 &  $0.0414$\scriptsize{$\pm 0.0070$}&  $9.3534$\scriptsize{$\pm 0.0927$}&    {$ 0.2510$\scriptsize{$\pm 0.0058$}} & 
$37$x\\
meps19  &  $0.0106$\scriptsize{$\pm 0.0010$}  &   {${3.3237}$\scriptsize{$\pm 0.0431$}} &  $0.0755$\scriptsize{$\pm 0.0037$} & 
$44$x\\
meps20  &  $0.0152$\scriptsize{$\pm 0.0016$}  &  $5.4003$\scriptsize{$\pm 0.3945$}   &   {$ 0.0948$\scriptsize{$\pm 0.0077$} } & 
$57$x\\
meps21 &  $0.0137$\scriptsize{$\pm 0.0008$} &  $4.1657$\scriptsize{$\pm 0.0670$}  &   {$ 0.0854$\scriptsize{$\pm 0.0146$} } & 
$49$x\\
star &  $0.0030$\scriptsize{$\pm 0.0006$} &   {$ {3.5842}$\scriptsize{$\pm 0.3722$}} &  $0.0332$\scriptsize{$\pm 0.0066$} & 
$108$x\\
bio &  $0.0291$\scriptsize{$\pm 0.0053$}&   {$ {7.5417}$\scriptsize{$\pm 1.1028$}} &  $0.2042$\scriptsize{$\pm 0.0344$} & 
$37$x\\
blog &  $0.0340$\scriptsize{$\pm 0.0024$}&   {$ {8.0913}$\scriptsize{$\pm 1.2072$}}&  $	0.2239$\scriptsize{$\pm 0.0261$} & 
$36$x\\
bike &  $0.0072$\scriptsize{$\pm 0.0007$}&  $ 3.5806$\scriptsize{$\pm 0.0285$}&   {$ 0.0534$\scriptsize{$\pm 0.0021$}} & 
$67$x\\
\bottomrule
\end{tabular}
\end{sc}
\end{small}
\end{center}
\end{table*}

\begin{table*}[t]
\caption {Comparison of coverage and band length among Vanilla CP, FCP, and FFCP. 
FFCP achieves significantly faster running speeds while performing comparably to FCP in most datasets and outperforming Vanilla CP.
For FFCP, we select the shortest band length among all layers.}
\label{tab1:1dim-results}
\begin{center}
\begin{small}
\begin{sc}
\resizebox{\linewidth}{!}{
\begin{tabular}{ccccccc}
\toprule
Method  &  \multicolumn{2}{c}{Vanilla CP}  &  \multicolumn{2}{c}{FCP} &  \multicolumn{2}{c}{FFCP} \\
\cmidrule(lr){2-3} \cmidrule(lr){4-5} \cmidrule(lr){6-7}
Dataset &  Coverage  &  Length  
&  Coverage  &  Length &  Coverage  &  Length \\
\midrule
\midrule
synthetic  
& $90.080$\scriptsize{$\pm 0.951$} & $0.176$\scriptsize{$\pm 0.015$} 
& $89.930$\scriptsize{$\pm 0.956$} & $\textbf{0.081}$\scriptsize{$\pm 0.041$} 
& $90.080$\scriptsize{$\pm 0.951$} & $0.176$\scriptsize{$\pm 0.015$} \\
com  
& $89.875$\scriptsize{$\pm 0.985$} & $1.974$\scriptsize{$\pm 0.071$} 
& $89.724$\scriptsize{$\pm 1.087$} & $1.939$\scriptsize{$\pm 1.408$} 
& $90.226$\scriptsize{$\pm 2.179$} & $\textbf{1.838}$\scriptsize{$\pm 0.180$}\\
fb1 
&  $90.254$\scriptsize{$\pm 0.170$} &  $2.004$\scriptsize{$\pm 0.191$}
&  $90.198$\scriptsize{$\pm 0.207$} &  $2.010$\scriptsize{$\pm 0.182$}
& $90.168$\scriptsize{$\pm 0.220$} & $\textbf{1.472}$\scriptsize{$\pm 0.232$}\\
fb2 
& $89.933$\scriptsize{$\pm 0.206$} & $2.016$\scriptsize{$\pm 0.218$}
& $89.966$\scriptsize{$\pm 0.130$} & $ \textbf{1.371}$\scriptsize{$\pm 0.370$}
& $89.868$\scriptsize{$\pm 0.062$} & $1.425$\scriptsize{$\pm 0.109$}\\
meps19  
& $90.567$\scriptsize{$\pm 0.311$} & $3.982$\scriptsize{$\pm 0.614$}
& $90.605$\scriptsize{$\pm 0.340$} & $3.493$\scriptsize{$\pm 2.734$} 
& $90.352$\scriptsize{$\pm 0.469$} & $\textbf{3.134}$\scriptsize{$\pm 0.309$} \\
meps20  
& $89.923$\scriptsize{$\pm 0.715$} & $4.184$\scriptsize{$\pm 0.316$}  
& $89.929$\scriptsize{$\pm 0.770$} & $\textbf{2.730}$\scriptsize{$\pm 0.962$} 
& $89.615$\scriptsize{$\pm 0.661$} & $3.268$\scriptsize{$\pm 0.283$}\\
meps21 
& $90.019$\scriptsize{$\pm 0.341$} & $3.732$\scriptsize{$\pm 0.555$} 
& $90.038$\scriptsize{$\pm 0.303$} & $3.393$\scriptsize{$\pm 1.313$} 
& $89.745$\scriptsize{$\pm 0.344$} & $\textbf{3.146}$\scriptsize{$\pm 0.506$}\\
star 
& $90.393$\scriptsize{$\pm 1.494$} & $0.208$\scriptsize{$\pm 0.004$} 
& $90.300$\scriptsize{$\pm 1.362$} & $\textbf{0.174}$\scriptsize{$\pm 0.038$}
& $90.393$\scriptsize{$\pm 1.494$} & $ 0.208$\scriptsize{$\pm 0.004$} \\
bio 
&  $89.875$\scriptsize{$\pm 0.488$} & $1.661$\scriptsize{$\pm 0.019$}
&  $89.930$\scriptsize{$\pm 0.501$} & $\textbf{1.412}$\scriptsize{$\pm 0.265$}
&  $89.875$\scriptsize{$\pm 0.488$} & $1.661$\scriptsize{$\pm 0.019$}\\
blog 
& $90.176$\scriptsize{$\pm 0.241$} & $3.524$\scriptsize{$\pm 0.850$}
& $90.151$\scriptsize{$\pm 0.405$} & $2.795$\scriptsize{$\pm 1.385$}
& $90.059$\scriptsize{$\pm 0.101$} & $\textbf{2.741}$\scriptsize{$\pm 0.517$}\\
bike 
& $89.871$\scriptsize$\pm 0.568$ & $0.703$\scriptsize{$\pm 0.016$}
& $89.394$\scriptsize$\pm 0.633$ & $2.147$\scriptsize$\pm 0.249$
& $89.624$\scriptsize{$\pm 0.688$} & $\textbf{0.635}$\scriptsize{$\pm 0.030$} \\
\bottomrule
\end{tabular}
}
\end{sc}
\end{small}
\end{center}
\end{table*}

\textbf{Runtime Comparison.}
The runtime comparison is presented in Table~\ref{tab: FFCP time}. 
The results show that FFCP outperforms FCP with an approximate 50x speedup in runtime. Notably, since Vanilla CP is the most basic method and does not utilize additional tools, it exhibits the fastest runtime.

\textbf{Coverage.}
Table~\ref{tab1:1dim-results} summarizes the coverage for the one-dimensional response.
Experimental results indicate that the coverage of FFCP all exceeds the confidence level $1 - \alpha$, affirming its effectiveness as stated in Theorem~\ref{thm: ffcp effective}.

\textbf{Band Length.}
The band length is detailed in Table~\ref{tab1:1dim-results} for a one-dimensional response.
It is noteworthy that FFCP surpasses Vanilla CP by achieving a shorter band length, thereby validating the efficiency of the algorithm.

\subsection{Extensions of FFCP}
\label{subsec:Extensions}
This section provides the extensions of FFCP, which is divided into two parts. Section~\ref{subsubsec:Other-task} mainly discusses the applications of FFCP beyond regression tasks, specifically in image classification~\citep{angelopoulos2020uncertainty} and segmentation tasks. Section~\ref{subsubsec: Extensions-ffcp} focuses on how the gradient-level techniques in FFCP can be extended to other CP variants, \eg, CQR~\citep{romano2019conformalized} and LCP~\citep{guan2023localized}.

\subsubsection{Other Tasks}
\label{subsubsec:Other-task}
\paragraph{Classification.} We extend the FFCP techniques to classification tasks using the baseline RAPS~\citep{angelopoulos2020uncertainty} model, creating a new variant called FFRAPS (Fast Feature RAPS, Algorithm~\ref{alg:ffraps} in Appendix~\ref{appendix: FFRAPS}). 
According to the experimental findings presented in Table~\ref{tab1:classification}, FFRAPS returns shorter band 
lengths while preserving the coverage compared to RAPS under most model structures.

\begin{table*}[t]
\caption{{Comparison of FFRAPS with the state-of-the-art method RAPS on Imagenet-Val. 
The FFRAPS method outperforms RAPS in most datasets.}}
\label{tab1:classification}
\begin{center}
\begin{small}
\resizebox{\linewidth}{!}{
\begin{sc}

\begin{tabular}{ccccccc}
\toprule
Method  &  \multicolumn{2}{c}{accuracy}   &  \multicolumn{2}{c}{coverage} & \multicolumn{2}{c}{length} \\
\cmidrule(lr){2-3} \cmidrule(lr){4-5} \cmidrule(lr){6-7} 
Model &  TOP-1  &  TOP-5  
&  RAPS  &  FFRAPS &  RAPS  &  FFRAPS  \\
\midrule
\midrule
ResNeXt101 
&$0.793$\scriptsize{$\pm 0.001$} & $0.945$\scriptsize{$\pm 0.001$} 
&$0.908$\scriptsize{$\pm 0.002$} & $0.907$\scriptsize{$\pm 0.002$} 
&$2.012$\scriptsize{$\pm 0.035$} & $\textbf{2.006}$\scriptsize{$\pm 0.039$} 
\\
ResNet152  
&$0.784$\scriptsize{$\pm 0.001$} & $0.941$\scriptsize{$\pm 0.001$} 
&$0.909$\scriptsize{$\pm 0.003$} & $0.907$\scriptsize{$\pm 0.003$} 
&$2.144$\scriptsize{$\pm 0.034$} & $\textbf{2.128}$\scriptsize{$\pm 0.058$} 
\\
ResNet101 
&$0.774$\scriptsize{$\pm 0.001$} & $0.935$\scriptsize{$\pm 0.001$} 
&$0.906$\scriptsize{$\pm 0.004$} & $0.906$\scriptsize{$\pm 0.003$} 
&$2.348$\scriptsize{$\pm 0.151$} & $\textbf{2.256}$\scriptsize{$\pm 0.037$} 
\\
ResNet50 
&$0.761$\scriptsize{$\pm 0.001$} & $0.929$\scriptsize{$\pm 0.001$} 
&$0.907$\scriptsize{$\pm 0.004$} & $0.907$\scriptsize{$\pm 0.003$} 
&$\textbf{2.560}$\scriptsize{$\pm 0.104$} & $2.594$\scriptsize{$\pm 0.069$} 
\\
ResNet18  
&$0.698$\scriptsize{$\pm 0.001$} & $0.891$\scriptsize{$\pm 0.001$} 
&$0.906$\scriptsize{$\pm 0.003$} & $0.903$\scriptsize{$\pm 0.003$} 
&$4.560$\scriptsize{$\pm 0.147$} & $\textbf{4.434}$\scriptsize{$\pm 0.168$} 
\\
DenseNet161 
&$0.772$\scriptsize{$\pm 0.001$} & $0.936$\scriptsize{$\pm 0.001$} 
&$0.907$\scriptsize{$\pm 0.003$} & $0.907$\scriptsize{$\pm 0.002$} 
&$2.374$\scriptsize{$\pm 0.083$} & $\textbf{2.328}$\scriptsize{$\pm 0.056$} 
\\
VGG16 
&$0.716$\scriptsize{$\pm 0.001$} & $0.904$\scriptsize{$\pm 0.001$} 
&$0.904$\scriptsize{$\pm 0.002$} & $0.902$\scriptsize{$\pm 0.002$} 
&$3.566$\scriptsize{$\pm 0.098$} & $\textbf{3.521}$\scriptsize{$\pm 0.065$} 
\\
Inception  
&$0.696$\scriptsize{$\pm 0.001$} & $0.887$\scriptsize{$\pm 0.001$} 
&$0.903$\scriptsize{$\pm 0.003$} & $0.903$\scriptsize{$\pm 0.002$} 
&$5.410$\scriptsize{$\pm 0.350$} & $\textbf{5.407}$\scriptsize{$\pm 0.133$} 
\\
ShuffleNet  
&$0.694$\scriptsize{$\pm 0.001$} & $0.883$\scriptsize{$\pm 0.001$} 
&$0.902$\scriptsize{$\pm 0.001$} & $0.901$\scriptsize{$\pm 0.002$} 
&$5.001$\scriptsize{$\pm 0.121$} & $\textbf{4.971}$\scriptsize{$\pm 0.073$} 
\\
\bottomrule
\end{tabular}

\end{sc}
}
\end{small}
\end{center}
\end{table*}

\paragraph{Segmentation.} The gradient-level techniques of FFCP also prove effective in segmentation tasks. 
The segmentation results in Figure~\ref{fig:ffcp-seg-visualization} reveal that FFCP returns appropriate bands across different regions. 
Specifically, larger bands are observed in less informative areas, such as at object boundaries, whereas narrower bands are found in more informative regions. 
This validates the efficiency of FFCP in segmentation tasks.

\subsubsection{Extending FFCP into Other Models}
\label{subsubsec: Extensions-ffcp}
\paragraph{Conformalized Quantile Regression (CQR,~\citet{romano2019conformalized})} The gradient-level techniques of FFCP are adaptable to other conformal prediction frameworks like CQR. 
We develop FFCQR (Fast Feature CQR, Algorithm~\ref{alg: ffcqr} in Appendix~\ref{appendix: FFCQR}), which not only significantly reduces runtime compared to FCQR but also exhibits better performance than CQR.
Additionally, we observe that for the neural network significant level setting $[\alpha, 1-\alpha]$ in the CQR method, as the $\alpha$ value increases, approaching $1-\alpha$, the performance of FFCQR gradually improves. For detailed experimental results in Table~\ref{tab:FFCQR-time-compare},\ref{tab:ffcqr-meps19}, \ref{tab:ffcqr-bike}, \ref{tab:ffcqr-com}, \ref{tab:ffcqr-bio} in the Appendix~\ref{appendix: FFCQR}.

\paragraph{Locally Adaptive Conformal Prediction (LCP,~\citet{guan2023localized})} 
Integrating gradient-level techniques from FFCP into the LCP method leads to FFLCP (Fast Feature LCP, Algorithm~\ref{alg: fflcp} in Appendix~\ref{appendix: FFLCP}).
Experimental results in Table~\ref{tab:group-coverage-LCP-FFLCP} indicate that FFLCP outperforms LCP in terms of group coverage, highlighting an improvement in the adaptability of LCP to locally adaptive methods.

\begin{table*}[t]
\caption{Comparison of LCP and FFLCP in group coverage. We divide the datasets into three groups based on the size of $Y$, and calculate the coverage for each group, returning the maximum coverage. FFLCP shows the results for the 5-layer neural network.}
\label{tab:group-coverage-LCP-FFLCP}
\begin{center}
\begin{small}
\begin{sc}
\resizebox{\linewidth}{!}{
\begin{tabular}{ccccccc}
\toprule
Method  &  \multicolumn{1}{c}{LCP}    & \multicolumn{5}{c}{FFLCP} \\
\cmidrule(lr){2-2}  \cmidrule(lr){3-7} 
Dataset &  Coverage    
&  layer0  &  layer1 &  layer2 &  layer3 &  layer4  \\
\midrule
\midrule
synthetic  &  $87.02$\scriptsize{$\pm 1.00$} & $86.93$\scriptsize{$\pm 0.78$} & 
$86.57$\scriptsize{$\pm 0.88$} 
& $85.43$\scriptsize{$\pm 1.24$}
& $\textbf{87.11}$\scriptsize{$\pm 1.76$}
&  $87.02$\scriptsize{$\pm 1.00$}
\\
com  & $80.33$\scriptsize{$\pm 3.24$} & $\textbf{81.84}$\scriptsize{$\pm 2.52$} & 
$79.56$\scriptsize{$\pm 3.34$} 
& $77.42$\scriptsize{$\pm 4.12$}
& $79.41$\scriptsize{$\pm 2.88$}
& $80.33$\scriptsize{$\pm 3.24$}

\\
fb1 & $52.51$\scriptsize{$\pm 1.76$}&  $\textbf{78.61}$\scriptsize{$\pm 0.91$}&   
$76.39$\scriptsize{$\pm 1.21$}&   
$ 67.16$\scriptsize{$\pm 1.61$}
& $57.82$\scriptsize{$\pm 1.88$}
& $52.51$\scriptsize{$\pm 1.76$}

\\
fb2 &  $54.33$\scriptsize{$\pm 1.75$}&  $\textbf{75.86}$\scriptsize{$\pm 0.83$}&
$75.44$\scriptsize{$\pm 0.91$}&   
$70.45$\scriptsize{$\pm 0.99$}&
$60.18$\scriptsize{$\pm 1.73$}&
$54.33$\scriptsize{$\pm 1.75$}

\\
meps19  & $67.35$\scriptsize{$\pm 1.21$}  
& $\textbf{68.19}$\scriptsize{$\pm 2.31$}
& $66.44$\scriptsize{$\pm 2.01$}   
& $64.94$\scriptsize{$\pm 1.53$} 
& $67.14$\scriptsize{$\pm 1.24$}
& $67.35$\scriptsize{$\pm 1.21$}

\\
meps20  & $65.49$\scriptsize{$\pm 1.64$}  &  $\textbf{69.30}$\scriptsize{$\pm 1.09$}  
& $68.80$\scriptsize{$\pm 1.55$}  
& $65.14$\scriptsize{$\pm 1.44$}
& $65.47$\scriptsize{$\pm 1.99$}
& $65.49$\scriptsize{$\pm 1.64$}

\\
meps21 & $66.38$\scriptsize{$\pm 0.95$} &  $67.82$\scriptsize{$\pm 1.10$} 
& $\textbf{67.96}$\scriptsize{$\pm 1.21$} 
& $66.21$\scriptsize{$\pm 1.33$} 
& $65.54$\scriptsize{$\pm 1.02$}
& $66.38$\scriptsize{$\pm 0.95$}

\\
star & $77.20$\scriptsize{$\pm 3.97$} 
&  $\textbf{79.69}$\scriptsize{$\pm 2.88$} 
&  $79.28$\scriptsize{$\pm 1.72$} 
&  $77.47$\scriptsize{$\pm 5.21$}
& $77.33$\scriptsize{$\pm 4.12$}
& $77.20$\scriptsize{$\pm 3.97$}

\\
bio &  $81.10$\scriptsize{$\pm 0.61$}
&  $86.33$\scriptsize{$\pm 0.51$}
&  $86.06$\scriptsize{$\pm 0.50$}
&  $\textbf{86.78}$\scriptsize{$\pm 0.57$}
&  $83.26$\scriptsize{$\pm 0.71$}
&  $81.10$\scriptsize{$\pm 0.61$}

\\
blog & $48.99$\scriptsize{$\pm 1.01$}
&  $\textbf{61.01}$\scriptsize{$\pm 0.82$}
&  $55.10$\scriptsize{$\pm 1.12$}
&  $46.01$\scriptsize{$\pm 0.65$}
& $46.88$\scriptsize{$\pm 0.81$}
& $48.99$\scriptsize{$\pm 1.01$}

\\
bike &  $77.61$\scriptsize{$\pm 1.52$}
&  $81.02$\scriptsize{$\pm 1.73$}
&  $82.42$\scriptsize{$\pm 2.08$}
&  $82.97$\scriptsize{$\pm 1.29$}
&  $\textbf{84.41}$\scriptsize{$\pm 1.71$}
&  $77.61$\scriptsize{$\pm 1.52$}

\\
\bottomrule
\end{tabular}
}
\end{sc}
\end{small}
\end{center}
\end{table*}

\subsection{Discussion}
\label{subsec:more-discussion}

\begin{table*}[t]
\caption{Coverage and Band Length based on Gradient from Different Layers of Neural Networks. FFCP LAYER(·) represents using the gradient between the LAYER(·) and the output. The results in LAYER4 are equivalent to Vanilla CP. }
\label{tab:ffcp-res-each-layer}
\begin{center}
\begin{small}
\begin{sc}
\resizebox{\linewidth}{!}{
\begin{tabular}{ccccccccc}
\toprule
layer    &  \multicolumn{2}{c}{layer1} &  \multicolumn{2}{c}{layer2}&  \multicolumn{2}{c}{layer3}
&  \multicolumn{2}{c}{layer4}\\
\cmidrule(lr){1-1} \cmidrule(lr){2-9}
Dataset   
&  Coverage  &  Length &  Coverage  &  Length
&  Coverage  &  Length &  Coverage  &  Length \\
\midrule
\midrule
synthetic  
& $89.810$\scriptsize{$\pm 0.784$} & $0.184$\scriptsize{$\pm 0.018$} 
& $90.050$\scriptsize{$\pm 0.534$} & $0.184$\scriptsize{$\pm 0.017$}
& $89.960$\scriptsize{$\pm 0.910$} & $\textbf{0.182}$\scriptsize{$\pm 0.023$}
& $90.220$\scriptsize{$\pm 0.983$} & $0.189$\scriptsize{$\pm 0.033$}\\
com  
& $90.476$\scriptsize{$\pm 1.889$} & $1.878$\scriptsize{$\pm 0.224$} 
& $90.226$\scriptsize{$\pm 2.179$} & $\textbf{1.838}$\scriptsize{$\pm 0.180$}
& $89.674$\scriptsize{$\pm 1.465$} & $1.853$\scriptsize{$\pm 0.136$}
& $89.825$\scriptsize{$\pm 0.646$} & $2.037$\scriptsize{$\pm 0.188$}\\
fb1 
& $90.112$\scriptsize{$\pm 0.199$} & $3.540$\scriptsize{$\pm 0.327$}
& $90.212$\scriptsize{$\pm 0.357$} & $2.860$\scriptsize{$\pm 0.327$}
& $90.083$\scriptsize{$\pm 0.216$} & $1.597$\scriptsize{$\pm 0.052$}
& $90.168$\scriptsize{$\pm 0.220$} & $\textbf{1.472}$\scriptsize{$\pm 0.232$}\\
fb2 
& $89.953$ \scriptsize{$\pm 0.250$}& $3.530$\scriptsize{$\pm 0.384$}
& $89.897$\scriptsize{$\pm 0.235$} & $3.048$\scriptsize{$\pm 0.510$}
& $89.956$\scriptsize{$\pm 0.159$} & $2.077$\scriptsize{$\pm 0.517$}
& $89.868$\scriptsize{$\pm 0.062$} & $\textbf{1.425}$\scriptsize{$\pm 0.109$}\\
meps19  
& $90.155$\scriptsize{$\pm 0.643$}  &  $3.251$\scriptsize{$\pm 0.396$} 
& $90.352$\scriptsize{$\pm 0.469$} & $\textbf{3.134}$\scriptsize{$\pm 0.309$} 
& $90.440$\scriptsize{$\pm 0.183$} & $3.184$\scriptsize{$\pm 0.482$}
& $90.586$\scriptsize{$\pm 0.246$} & $3.795$\scriptsize{$\pm 0.640$}\\
meps20  
& $89.934$\scriptsize{$\pm 0.520$}  & $4.302$\scriptsize{$\pm 1.377$}
& $89.889$\scriptsize{$\pm 0.621$} & $3.573$\scriptsize{$\pm 0.488$}
& $89.615$\scriptsize{$\pm 0.661$} & $\textbf{3.268}$\scriptsize{$\pm 0.283$}
& $89.82$\scriptsize{$\pm 0.689$} & $3.817$\scriptsize{$\pm	0.308$}\\
meps21 
& $89.496$\scriptsize{$\pm 0.262$} & $3.443$\scriptsize{$\pm 0.487$}
& $89.623$\scriptsize{$\pm 0.275$} & $3.218$\scriptsize{$\pm 0.239$}
& $89.745$\scriptsize{$\pm 0.344$} & $\textbf{3.146}$\scriptsize{$\pm 0.506$}
& $90.026$\scriptsize{$\pm 0.301$} & $3.452$\scriptsize{$\pm 0.711$}\\
star  
& $90.901$\scriptsize{$\pm 1.732$} &  $0.221$\scriptsize{$\pm 0.002$} 
& $90.993$\scriptsize{$\pm 1.807$} & $0.217$\scriptsize{$\pm 0.003$}
& $91.039$\scriptsize{$\pm 1.442$} & $0.210$\scriptsize{$\pm 0.004$}
& $90.300$\scriptsize{$\pm 1.248$} & $\textbf{0.209}$\scriptsize{$\pm 0.004$}\\
bio 
& $89.937$\scriptsize{$\pm 0.391$}&  $2.292$\scriptsize{$\pm 0.077$}
& $90.022$\scriptsize{$\pm 0.375$} & $2.042$\scriptsize{$\pm 0.067$}
& $89.991$\scriptsize{$\pm 0.594$} & ${2.080}$\scriptsize{$\pm 0.063$}
& $90.127$\scriptsize{$\pm 0.476$} & $\textbf{1.822}$\scriptsize{$\pm 0.025$}\\
blog 
& $89.968$\scriptsize{$\pm 0.420$}&  $4.772$\scriptsize{$\pm 0.614$}
& $89.918$\scriptsize{$\pm 0.319$} & $3.404$\scriptsize{$\pm 0.598$}
& $90.059$\scriptsize{$\pm 0.101$} & $\textbf{2.741}$\scriptsize{$\pm 0.517$}
& $90.017$\scriptsize{$\pm 0.197$} & $3.058$\scriptsize{$\pm 0.873$}\\
bike 
& $89.917$\scriptsize{$\pm 0.791$} & $1.701$\scriptsize{$\pm 0.254$}
& $89.568$\scriptsize{$\pm 0.476$} & $1.138$\scriptsize{$\pm 0.114$}
& $89.495$\scriptsize{$\pm 0.579$} & $0.794$\scriptsize{$\pm 0.068$}
& $89.624$\scriptsize{$\pm 0.688$} & $\textbf{0.635}$\scriptsize{$\pm 0.030$}\\
\bottomrule
\end{tabular}}
\end{sc}
\end{small}
\end{center}
\end{table*}

\textbf{Robustness for FFCP.} The empirical performance of FFCP demonstrates its robustness, as seen in the ablation studies on splitting points. 
We demonstrate that coverage remains robust across different splitting points in neural networks, as detailed in Table~\ref{tab:layer-ablation} in Appendix~\ref{appendix: Robustness}.
Furthermore, the results from different layers of the FFCP network are consistent, as presented in Table~\ref{tab:ffcp-res-each-layer} 

\textbf{FFCP is similar to FCP.}
We compare the relationship between the scores of FCP and FFCP through experiments. Figure~\ref{fig:FFCP-to-FCP} indicates a positive correlation between the non-conformity scores of the two algorithms, suggesting that  FFCP shares similarities with FCP in score function.

\textbf{FFCP on untrained network.}
We propose that FFCP returns shorter band lengths through its deployment of deep representations from the gradients.
To test this view, we contrast FFCP's performance using an untrained neural network against a baseline model.
Using an incompletely trained neural network, FFCP's performance deteriorates and becomes comparable to that of Vanilla CP. 
This is due to the partially incorrect semantic information in the gradient, which \emph{misleads} FFCP.
We defer the results to Table~\ref{tab1:untrain-model} and related discussion to Appendix~\ref{appendix: fcp advantage untrained not work}.
\begin{table*}[t]
\caption{Untrained model comparison between Vanilla CP and FFCP. When the model has not been sufficiently trained, FFCP performs similarly to Vanilla CP. This means that the model's performance determines the quality of the feature information in the gradient. When the model performs poorly, the gradient information obtained by FFCP is inaccurate. On the other hand, this also suggests that FFCP effectively utilizes the feature information in the gradient when the model is well-trained.}

\label{tab1:untrain-model}
\begin{center}
\begin{small}
\begin{sc}
\begin{tabular}{ccccc}
\toprule
Method  &  \multicolumn{2}{c}{Vanilla CP}   &  \multicolumn{2}{c}{FFCP} \\
\cmidrule(lr){2-3} \cmidrule(lr){4-5} 
Dataset &  Coverage  &  Length  
&  Coverage  &  Length \\
\midrule
\midrule
synthetic  &  $90.23$\scriptsize{$\pm 0.45$} & $\textbf{2.34}$\scriptsize{$\pm 0.01$} & 
$90.22$\scriptsize{$\pm 0.96$} 
& $2.41$\scriptsize{$\pm 0.01$} \\
com  &  $90.33$\scriptsize{$\pm 1.81$} & ${4.86}$\scriptsize{$\pm 0.13$} & 
$90.43$\scriptsize{$\pm 1.99$} 
& $\textbf{4.73}$\scriptsize{$\pm 0.08$} \\
fb1 &  $90.18$\scriptsize{$\pm 0.19$}&  $3.57$\scriptsize{$\pm 0.09$}&  $90.10$\scriptsize{$\pm 0.13$}&   {$ \textbf{3.57}$\scriptsize{$\pm 0.08$}}\\
fb2 &  $90.16$\scriptsize{$\pm 0.11$}&  $3.66$\scriptsize{$\pm 0.11$}&  $90.12$\scriptsize{$\pm 0.14$}&   {$ \textbf{3.66}$\scriptsize{$\pm 0.06$}}\\
meps19  &  $90.80$\scriptsize{$\pm 0.43$}  &   {$ \textbf{4.33}$\scriptsize{$\pm 0.07$}}&     $90.85$\scriptsize{$\pm 0.58$}   &  $4.38$\scriptsize{$\pm 0.07$} \\
meps20  &  $90.15$\scriptsize{$\pm 0.55$}  &  $\textbf{4.41}$\scriptsize{$\pm 0.23$}  &  $90.27$\scriptsize{$\pm 0.63$}  &   {$4.46$\scriptsize{$\pm 0.25$} }\\
meps21 &  $89.80$\scriptsize{$\pm 0.45$} &  $\textbf{4.41}$\scriptsize{$\pm 0.17$} &  $	89.89$\scriptsize{$\pm 0.56$} &   $\textbf{4.41}$\scriptsize{$\pm 0.15$} \\
star &  $89.79$\scriptsize{$\pm 0.51$} &   {$ \textbf{1.88}$\scriptsize{$\pm 0.01$}} &  $89.98$\scriptsize{$\pm 0.56$} &  $1.94$\scriptsize{$\pm 0.01$}\\
bio &  $90.16$\scriptsize{$\pm 0.20$}&   {$ {4.09}$\scriptsize{$\pm 0.02$}}&  $90.07$\scriptsize{$\pm 0.14$}&  $\textbf{4.04}$\scriptsize{$\pm 0.02$}\\
blog &  $90.11$\scriptsize{$\pm 0.30$}&   {$ \textbf{2.53}$\scriptsize{$\pm 0.12$}}&  $90.12$\scriptsize{$\pm 0.28$}&  $2.55$\scriptsize{$\pm 0.14$}\\
bike &  $89.55$\scriptsize{$\pm 0.82$}&  $\textbf{4.56}$\scriptsize{$\pm 0.09$}&  $89.57$\scriptsize{$\pm 0.86$}&   {$4.60$\scriptsize{$\pm 0.10$}}\\
\bottomrule
\end{tabular}
\end{sc}
\end{small}
\end{center}
\end{table*}
\begin{figure}[t]
    \centering
    \subfigure{\includegraphics[height=0.4\linewidth]{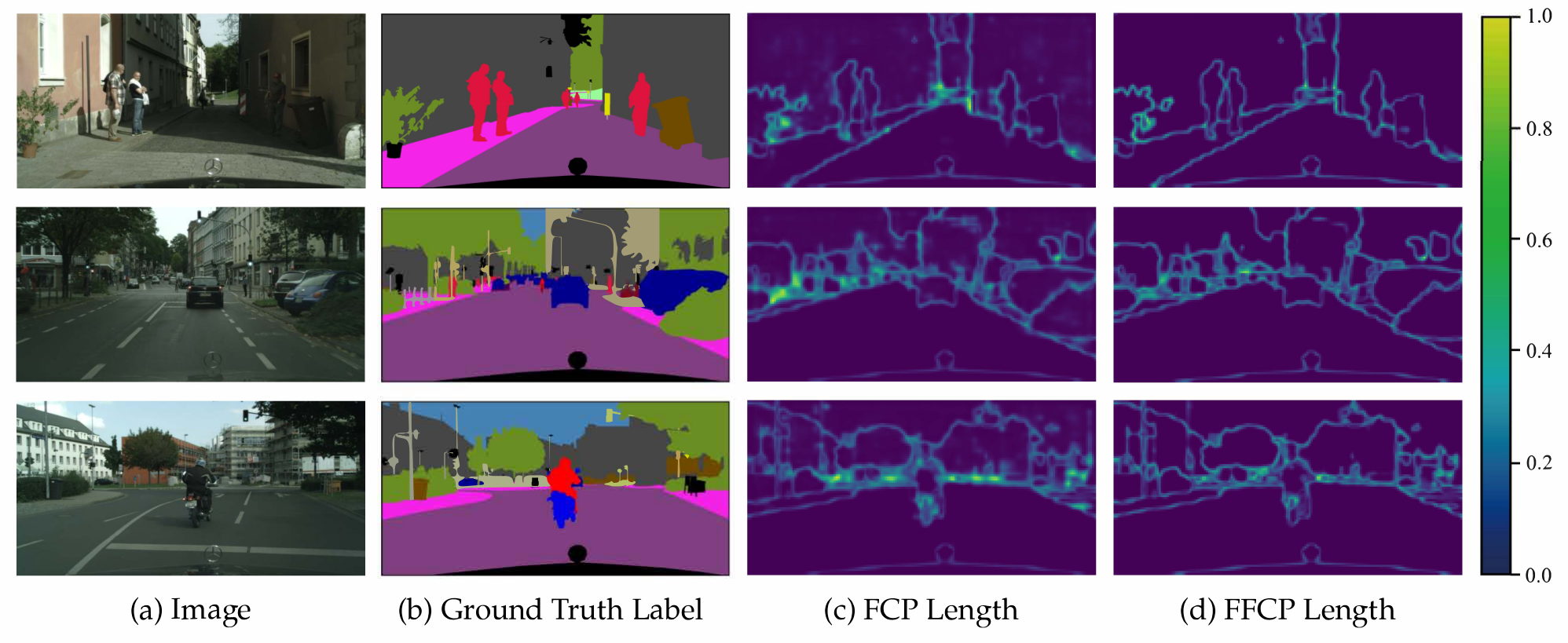}}
    \vspace{-0.2cm}
    \caption{The results of FFCP and FCP in image segmentation tasks show that brighter regions indicate areas of uncertainty. Both FFCP and FCP highlight uncertain regions around the edges, which is informative. FFCP, however, returns bands with sharper boundaries.}
    \label{fig:ffcp-seg-visualization}
\end{figure}
\begin{figure}[t]
    \centering
    \subfigure[Layer1]{\includegraphics[height=0.23\linewidth]{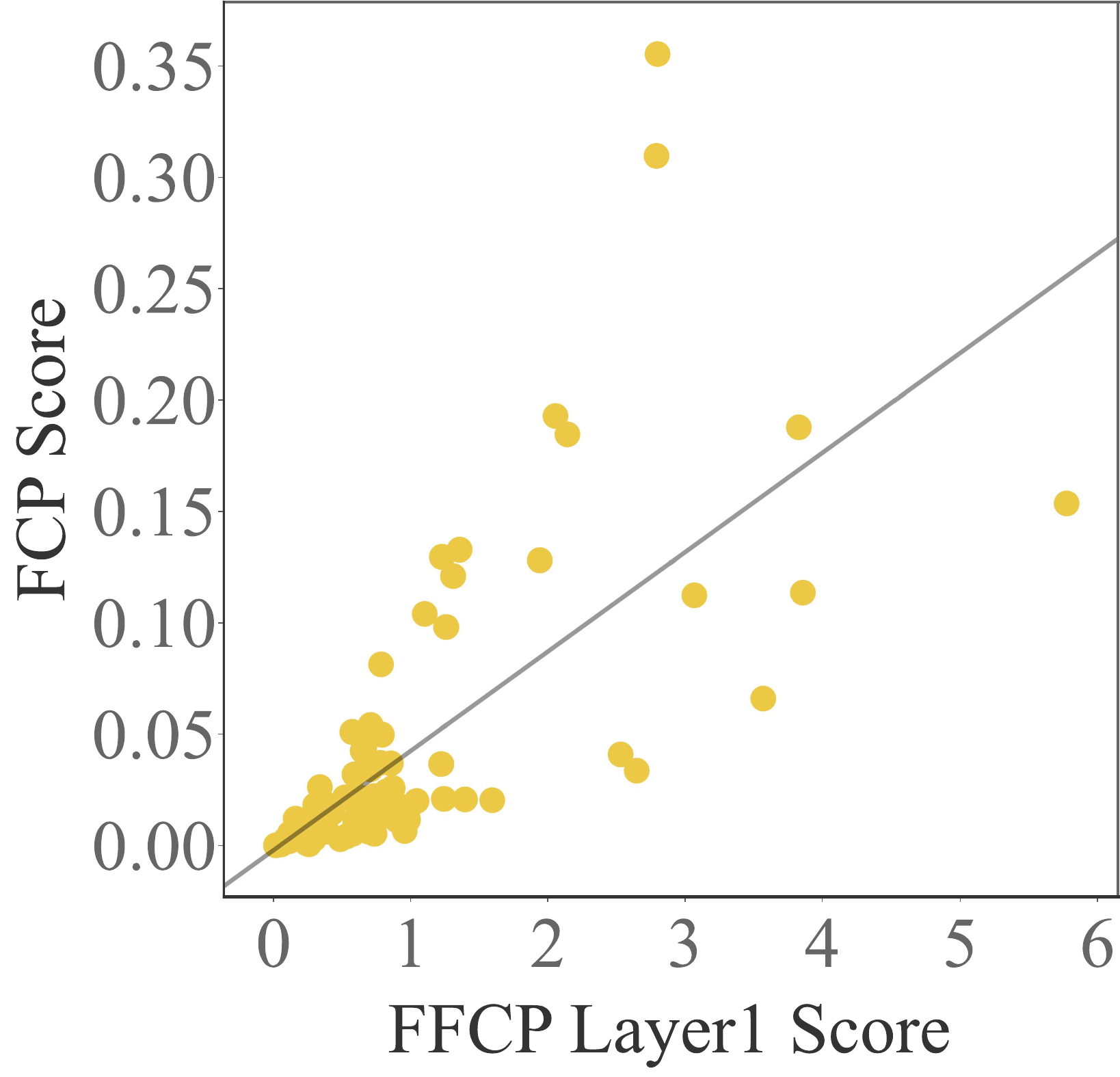}}
    \subfigure[Layer2]{\includegraphics[height=0.23\linewidth]{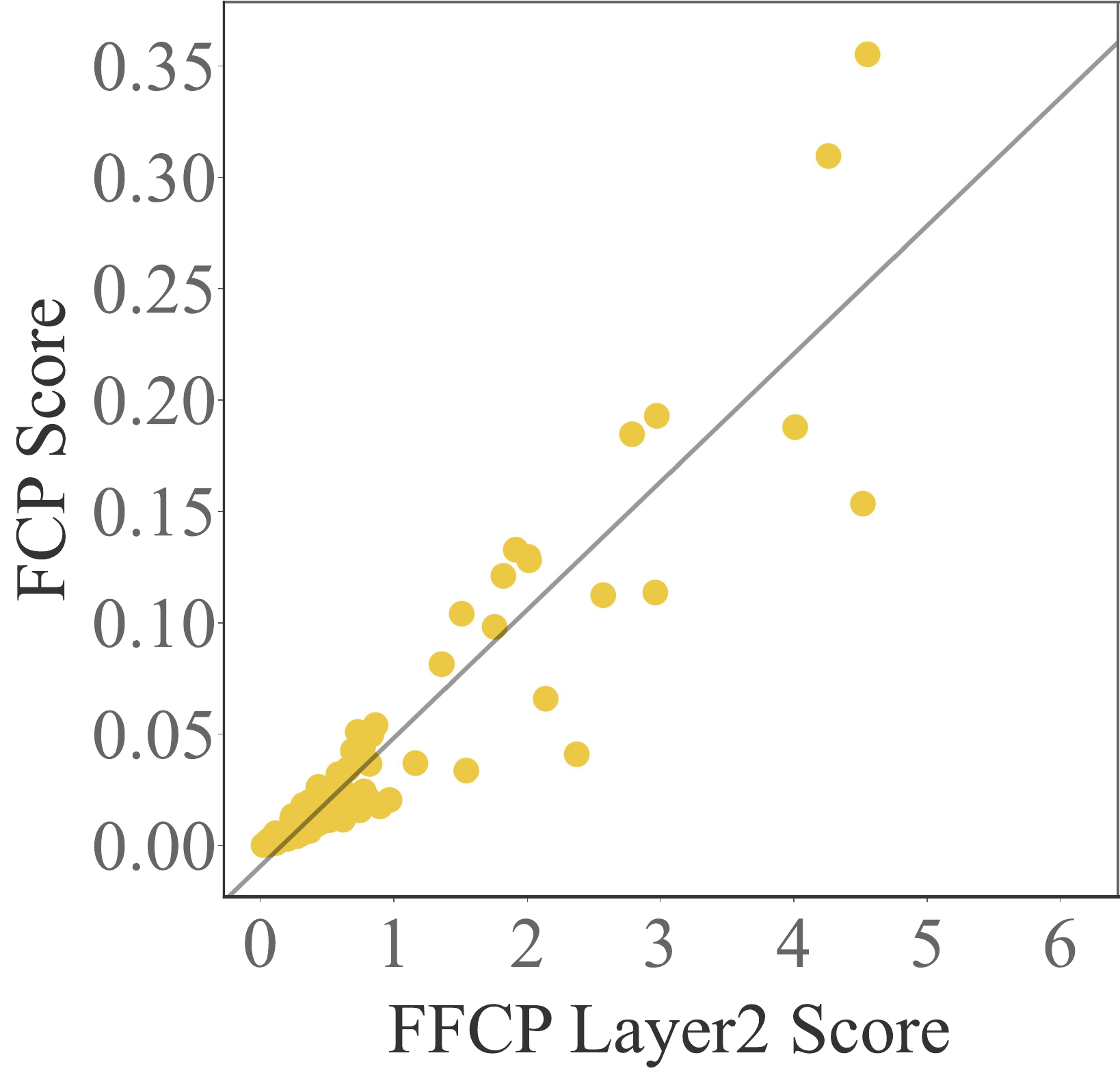}}
    \subfigure[Layer3]{\includegraphics[height=0.23\linewidth]{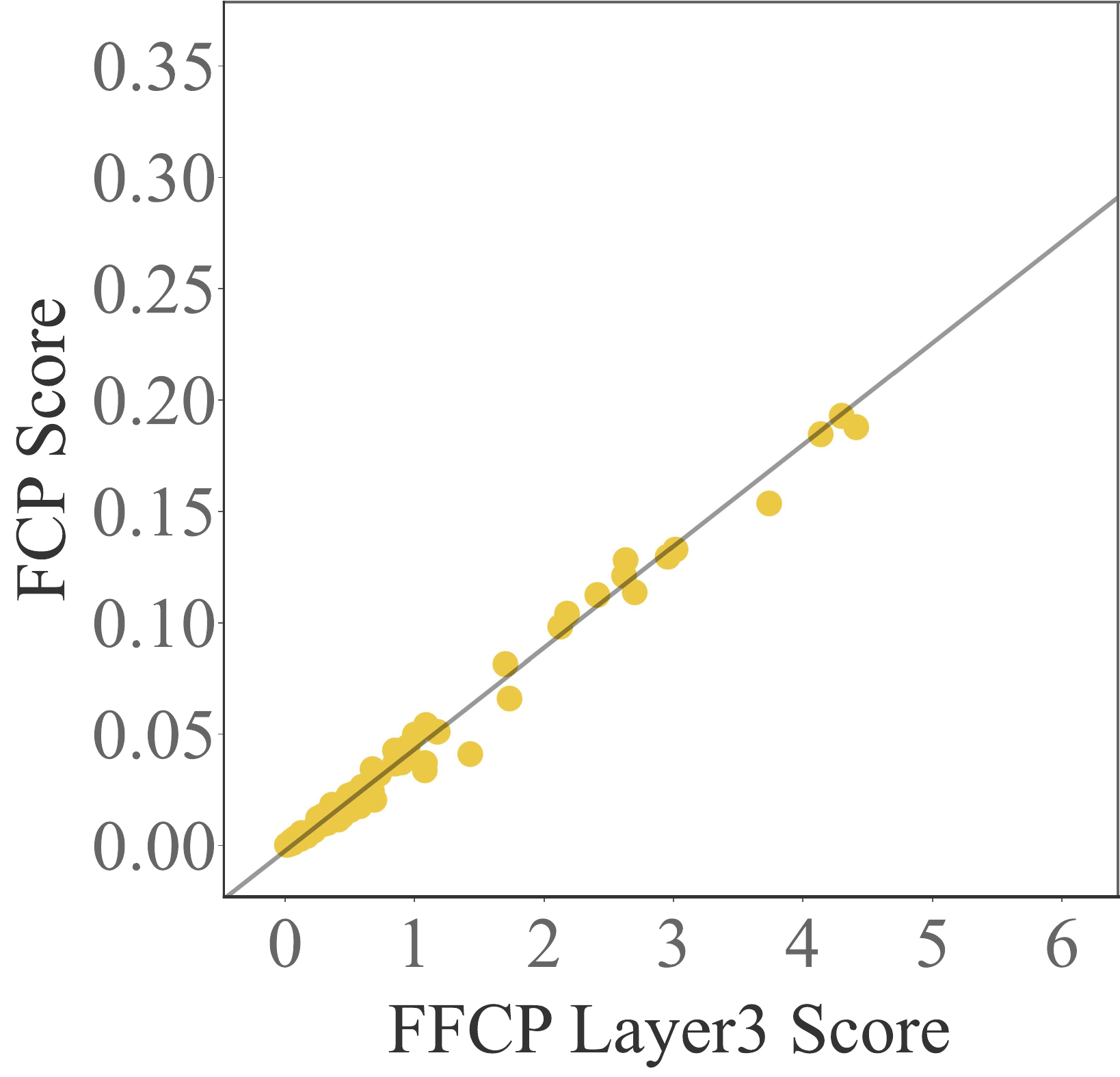}}
    \subfigure[Layer4]{\includegraphics[height=0.23\linewidth]{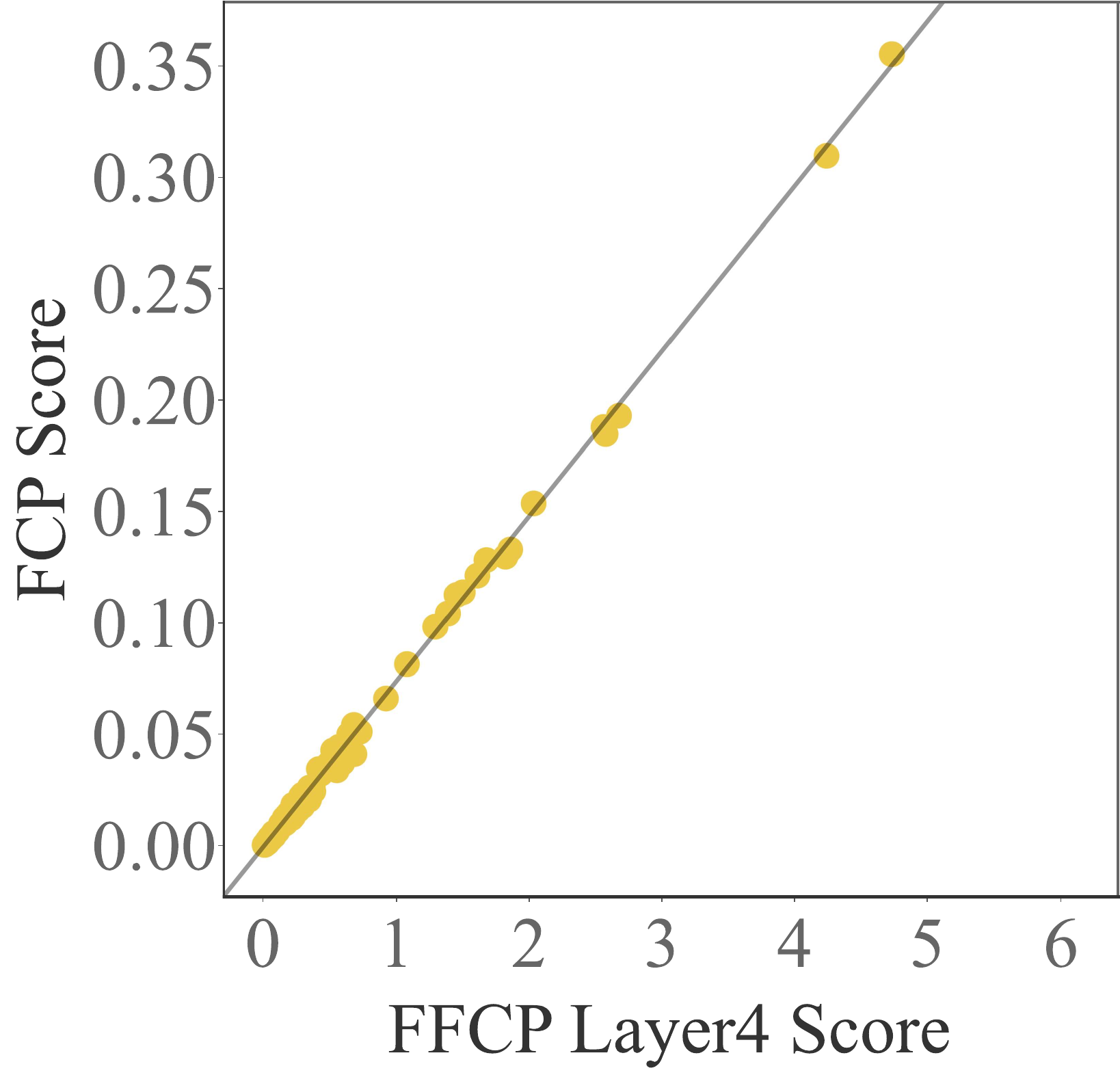}}
    \vspace{-0.2cm}
    \caption{Scatter plot of FCP Score and FFCP Score at different Layers.  The relationship between FCP Score and FFCP Score is positively correlated, which indicates that FFCP Score effectively replaces FCP Score.}
    \label{fig:FFCP-to-FCP}
\end{figure}

\section{Conclusion}
\label{sec:conclusion}
In this paper, we propose FFCP, a gradient-based non-conformity score that serves as a faster alternative to FCP, achieving processing times 50 times faster than FCP. Theoretically, we have established the effectiveness and efficiency of FFCP under mild assumptions. 
In our experiments, we apply FFCP to a variety of tasks, including basic regression tasks, classification tasks with FFRAPS based on RAPS, and segmentation. Additionally, we introduce FFCQR based on CQR, as well as FFLCP based on LCP.
These experimental results demonstrate the broad adaptability of our techniques.

For future work, the following points could be considered: (1) We use information from the first derivative and have not delved into higher-order derivatives, which may contain more feature information; (2) The computation of gradients becomes very costly as dimensionality increases, thus special methods are needed to handle large-scale tasks; and (3) The gradient at a single point may be unstable, especially when the gradient is zero, so methods such as random smoothing could be considered.

\clearpage
\bibliography{reference}

\begin{thebibliography}{65}
\providecommand{\natexlab}[1]{#1}
\providecommand{\url}[1]{\texttt{#1}}
\expandafter\ifx\csname urlstyle\endcsname\relax
  \providecommand{\doi}[1]{doi: #1}\else
  \providecommand{\doi}{doi: \begingroup \urlstyle{rm}\Url}\fi

\bibitem[Jordan and Mitchell(2015)]{jordan2015machine}
Michael~I Jordan and Tom~M Mitchell.
\newblock Machine learning: Trends, perspectives, and prospects.
\newblock \emph{Science}, 349\penalty0 (6245):\penalty0 255--260, 2015.

\bibitem[Silver et~al.(2017)Silver, Schrittwieser, Simonyan, Antonoglou, Huang, Guez, Hubert, Baker, Lai, Bolton, et~al.]{silver2017mastering}
David Silver, Julian Schrittwieser, Karen Simonyan, Ioannis Antonoglou, Aja Huang, Arthur Guez, Thomas Hubert, Lucas Baker, Matthew Lai, Adrian Bolton, et~al.
\newblock Mastering the game of go without human knowledge.
\newblock \emph{nature}, 550\penalty0 (7676):\penalty0 354--359, 2017.

\bibitem[Wei et~al.(2022)Wei, Xie, Cheng, Feng, An, and Li]{wei2022mitigating}
Hongxin Wei, Renchunzi Xie, Hao Cheng, Lei Feng, Bo~An, and Yixuan Li.
\newblock Mitigating neural network overconfidence with logit normalization.
\newblock In \emph{International conference on machine learning}, pages 23631--23644. PMLR, 2022.

\bibitem[Ji et~al.(2023)Ji, Lee, Frieske, Yu, Su, Xu, Ishii, Bang, Madotto, and Fung]{ji2023survey}
Ziwei Ji, Nayeon Lee, Rita Frieske, Tiezheng Yu, Dan Su, Yan Xu, Etsuko Ishii, Ye~Jin Bang, Andrea Madotto, and Pascale Fung.
\newblock Survey of hallucination in natural language generation.
\newblock \emph{ACM Computing Surveys}, 55\penalty0 (12):\penalty0 1--38, 2023.

\bibitem[Gelijns et~al.(2001)Gelijns, Zivin, and Nelson]{gelijns2001uncertainty}
Annetine~C Gelijns, Joshua~Graff Zivin, and Richard~R Nelson.
\newblock Uncertainty and technological change in medicine.
\newblock \emph{Journal of Health Politics, Policy and Law}, 26\penalty0 (5):\penalty0 913--924, 2001.

\bibitem[Thirumurthy et~al.(2019)Thirumurthy, Asch, and Volpp]{thirumurthy2019uncertain}
Harsha Thirumurthy, David~A Asch, and Kevin~G Volpp.
\newblock The uncertain effect of financial incentives to improve health behaviors.
\newblock \emph{Jama}, 321\penalty0 (15):\penalty0 1451--1452, 2019.

\bibitem[Morduch and Schneider(2017)]{morduch2017financial}
Jonathan Morduch and Rachel Schneider.
\newblock \emph{The financial diaries: How American families cope in a world of uncertainty}.
\newblock Princeton University Press, 2017.

\bibitem[Abdar et~al.(2021)Abdar, Pourpanah, Hussain, Rezazadegan, Liu, Ghavamzadeh, Fieguth, Cao, Khosravi, Acharya, et~al.]{abdar2021review}
Moloud Abdar, Farhad Pourpanah, Sadiq Hussain, Dana Rezazadegan, Li~Liu, Mohammad Ghavamzadeh, Paul Fieguth, Xiaochun Cao, Abbas Khosravi, U~Rajendra Acharya, et~al.
\newblock A review of uncertainty quantification in deep learning: Techniques, applications and challenges.
\newblock \emph{Information fusion}, 76:\penalty0 243--297, 2021.

\bibitem[Guo et~al.(2017)Guo, Pleiss, Sun, and Weinberger]{Guo2017OnCO}
Chuan Guo, Geoff Pleiss, Yu~Sun, and Kilian~Q. Weinberger.
\newblock On calibration of modern neural networks.
\newblock \emph{ArXiv}, abs/1706.04599, 2017.

\bibitem[Chen et~al.(2021)Chen, Zhang, Gutmann, Courville, and Zhu]{chen2021neural}
Yanzhi Chen, Dinghuai Zhang, Michael~U. Gutmann, Aaron Courville, and Zhanxing Zhu.
\newblock Neural approximate sufficient statistics for implicit models.
\newblock In \emph{International Conference on Learning Representations}, 2021.
\newblock URL \url{https://openreview.net/forum?id=SRDuJssQud}.

\bibitem[Gawlikowski et~al.(2021)Gawlikowski, Tassi, Ali, Lee, Humt, Feng, Kruspe, Triebel, Jung, Roscher, Shahzad, Yang, Bamler, and Zhu]{DBLP:journals/corr/abs-2107-03342}
Jakob Gawlikowski, Cedrique Rovile~Njieutcheu Tassi, Mohsin Ali, Jongseok Lee, Matthias Humt, Jianxiang Feng, Anna~M. Kruspe, Rudolph Triebel, Peter Jung, Ribana Roscher, Muhammad Shahzad, Wen Yang, Richard Bamler, and Xiao~Xiang Zhu.
\newblock A survey of uncertainty in deep neural networks.
\newblock \emph{CoRR}, abs/2107.03342, 2021.
\newblock URL \url{https://arxiv.org/abs/2107.03342}.

\bibitem[Vovk et~al.(2005)Vovk, Gammerman, and Shafer]{vovk2005algorithmic}
Vladimir Vovk, Alexander Gammerman, and Glenn Shafer.
\newblock \emph{Algorithmic learning in a random world}, volume~29.
\newblock Springer, 2005.

\bibitem[Shafer and Vovk(2008{\natexlab{a}})]{shafer2008tutorial}
Glenn Shafer and Vladimir Vovk.
\newblock A tutorial on conformal prediction.
\newblock \emph{Journal of Machine Learning Research}, 9\penalty0 (3), 2008{\natexlab{a}}.

\bibitem[Burnaev and Vovk(2014)]{burnaev2014efficiency}
Evgeny Burnaev and Vladimir Vovk.
\newblock Efficiency of conformalized ridge regression.
\newblock In \emph{Conference on Learning Theory}, pages 605--622. PMLR, 2014.

\bibitem[Shen et~al.(2014)Shen, He, Gao, Deng, and Mesnil]{shen2014learning}
Yelong Shen, Xiaodong He, Jianfeng Gao, Li~Deng, and Gr{\'e}goire Mesnil.
\newblock Learning semantic representations using convolutional neural networks for web search.
\newblock In \emph{Proceedings of the 23rd international conference on world wide web}, pages 373--374, 2014.

\bibitem[Teng et~al.(2022)Teng, Wen, Zhang, Bengio, Gao, and Yuan]{teng2022predictive}
Jiaye Teng, Chuan Wen, Dinghuai Zhang, Yoshua Bengio, Yang Gao, and Yang Yuan.
\newblock Predictive inference with feature conformal prediction.
\newblock \emph{arXiv preprint arXiv:2210.00173}, 2022.

\bibitem[Xu et~al.(2020)Xu, Shi, Zhang, Wang, Chang, Huang, Kailkhura, Lin, and Hsieh]{xu2020automatic}
Kaidi Xu, Zhouxing Shi, Huan Zhang, Yihan Wang, Kai-Wei Chang, Minlie Huang, Bhavya Kailkhura, Xue Lin, and Cho-Jui Hsieh.
\newblock Automatic perturbation analysis for scalable certified robustness and beyond.
\newblock \emph{Advances in Neural Information Processing Systems}, 33:\penalty0 1129--1141, 2020.

\bibitem[Romano et~al.(2019{\natexlab{a}})Romano, Patterson, and Candes]{romano2019conformalized}
Yaniv Romano, Evan Patterson, and Emmanuel Candes.
\newblock Conformalized quantile regression.
\newblock \emph{Advances in neural information processing systems}, 32, 2019{\natexlab{a}}.

\bibitem[Guan(2023)]{guan2023localized}
Leying Guan.
\newblock Localized conformal prediction: A generalized inference framework for conformal prediction.
\newblock \emph{Biometrika}, 110\penalty0 (1):\penalty0 33--50, 2023.

\bibitem[Angelopoulos et~al.(2020)Angelopoulos, Bates, Malik, and Jordan]{angelopoulos2020uncertainty}
Anastasios Angelopoulos, Stephen Bates, Jitendra Malik, and Michael~I Jordan.
\newblock Uncertainty sets for image classifiers using conformal prediction.
\newblock \emph{arXiv preprint arXiv:2009.14193}, 2020.

\bibitem[Shafer and Vovk(2008{\natexlab{b}})]{DBLP:journals/jmlr/ShaferV08}
Glenn Shafer and Vladimir Vovk.
\newblock A tutorial on conformal prediction.
\newblock \emph{J. Mach. Learn. Res.}, 9:\penalty0 371--421, 2008{\natexlab{b}}.
\newblock URL \url{https://dl.acm.org/citation.cfm?id=1390693}.

\bibitem[Barber et~al.(2020)Barber, Cand{\`e}s, Ramdas, and Tibshirani]{Barber2020TheLO}
Rina~Foygel Barber, Emmanuel~J. Cand{\`e}s, Aaditya Ramdas, and Ryan~J. Tibshirani.
\newblock The limits of distribution-free conditional predictive inference.
\newblock \emph{Information and Inference: A Journal of the IMA}, 2020.

\bibitem[Ye et~al.(2024)Ye, Yang, Pang, Wang, Wong, Yilmaz, Shi, and Tu]{ye2024benchmarking}
Fanghua Ye, Mingming Yang, Jianhui Pang, Longyue Wang, Derek~F Wong, Emine Yilmaz, Shuming Shi, and Zhaopeng Tu.
\newblock Benchmarking llms via uncertainty quantification.
\newblock \emph{arXiv preprint arXiv:2401.12794}, 2024.

\bibitem[Kumar et~al.(2023)Kumar, Lu, Gupta, Palepu, Bellamy, Raskar, and Beam]{kumar2023conformal}
Bhawesh Kumar, Charlie Lu, Gauri Gupta, Anil Palepu, David Bellamy, Ramesh Raskar, and Andrew Beam.
\newblock Conformal prediction with large language models for multi-choice question answering.
\newblock \emph{arXiv preprint arXiv:2305.18404}, 2023.

\bibitem[Quach et~al.(2023)Quach, Fisch, Schuster, Yala, Sohn, Jaakkola, and Barzilay]{quach2023conformal}
Victor Quach, Adam Fisch, Tal Schuster, Adam Yala, Jae~Ho Sohn, Tommi~S Jaakkola, and Regina Barzilay.
\newblock Conformal language modeling.
\newblock \emph{arXiv preprint arXiv:2306.10193}, 2023.

\bibitem[Tibshirani et~al.(2019{\natexlab{a}})Tibshirani, Barber, Cand{\`{e}}s, and Ramdas]{DBLP:conf/nips/TibshiraniBCR19}
Ryan~J. Tibshirani, Rina~Foygel Barber, Emmanuel~J. Cand{\`{e}}s, and Aaditya Ramdas.
\newblock Conformal prediction under covariate shift.
\newblock In Hanna~M. Wallach, Hugo Larochelle, Alina Beygelzimer, Florence d'Alch{\'{e}}{-}Buc, Emily~B. Fox, and Roman Garnett, editors, \emph{Advances in Neural Information Processing Systems 32: Annual Conference on Neural Information Processing Systems 2019, NeurIPS 2019, December 8-14, 2019, Vancouver, BC, Canada}, pages 2526--2536, 2019{\natexlab{a}}.
\newblock URL \url{https://proceedings.neurips.cc/paper/2019/hash/8fb21ee7a2207526da55a679f0332de2-Abstract.html}.

\bibitem[Hu and Lei(2020)]{Hu2020ADT}
Xiaoyu Hu and Jing Lei.
\newblock A distribution-free test of covariate shift using conformal prediction.
\newblock \emph{arXiv: Methodology}, 2020.

\bibitem[Podkopaev and Ramdas(2021)]{Podkopaev2021DistributionfreeUQ}
Aleksandr Podkopaev and Aaditya Ramdas.
\newblock Distribution-free uncertainty quantification for classification under label shift.
\newblock In \emph{UAI}, 2021.

\bibitem[Barber et~al.(2022)Barber, Candes, Ramdas, and Tibshirani]{barber2022conformal}
Rina~Foygel Barber, Emmanuel~J Candes, Aaditya Ramdas, and Ryan~J Tibshirani.
\newblock Conformal prediction beyond exchangeability.
\newblock \emph{arXiv preprint arXiv:2202.13415}, 2022.

\bibitem[Romano et~al.(2020)Romano, Sesia, and Cand{\`e}s]{Romano2020ClassificationWV}
Yaniv Romano, Matteo Sesia, and Emmanuel~J. Cand{\`e}s.
\newblock Classification with valid and adaptive coverage.
\newblock \emph{arXiv: Methodology}, 2020.

\bibitem[Xu and Xie(2021)]{Xu2021ConformalPI}
Chen Xu and Yao Xie.
\newblock Conformal prediction interval for dynamic time-series.
\newblock In \emph{ICML}, 2021.

\bibitem[Gibbs and Cand{\`e}s(2021)]{Gibbs2021AdaptiveCI}
Isaac Gibbs and Emmanuel~J. Cand{\`e}s.
\newblock Adaptive conformal inference under distribution shift.
\newblock In \emph{NeurIPS}, 2021.

\bibitem[Teng et~al.(2021)Teng, Tan, and Yuan]{DBLP:conf/icml/TengTY21}
Jiaye Teng, Zeren Tan, and Yang Yuan.
\newblock {T-SCI:} {A} two-stage conformal inference algorithm with guaranteed coverage for cox-mlp.
\newblock In Marina Meila and Tong Zhang, editors, \emph{Proceedings of the 38th International Conference on Machine Learning, {ICML} 2021, 18-24 July 2021, Virtual Event}, volume 139 of \emph{Proceedings of Machine Learning Research}, pages 10203--10213. {PMLR}, 2021.
\newblock URL \url{http://proceedings.mlr.press/v139/teng21a.html}.

\bibitem[Cand{\`e}s et~al.(2023)Cand{\`e}s, Lei, and Ren]{candes2023conformalized}
Emmanuel Cand{\`e}s, Lihua Lei, and Zhimei Ren.
\newblock Conformalized survival analysis.
\newblock \emph{Journal of the Royal Statistical Society Series B: Statistical Methodology}, 85\penalty0 (1):\penalty0 24--45, 2023.

\bibitem[Cand{\`e}s et~al.(2021)Cand{\`e}s, Lei, and Ren]{candes2021conformalized}
Emmanuel~J Cand{\`e}s, Lihua Lei, and Zhimei Ren.
\newblock Conformalized survival analysis.
\newblock \emph{arXiv preprint arXiv:2103.09763}, 2021.

\bibitem[Lei et~al.(2013)Lei, Rinaldo, and Wasserman]{Lei2013ACP}
Jing Lei, Alessandro Rinaldo, and Larry~A. Wasserman.
\newblock A conformal prediction approach to explore functional data.
\newblock \emph{Annals of Mathematics and Artificial Intelligence}, 74:\penalty0 29--43, 2013.

\bibitem[Yang et~al.(2024)Yang, Cand{\`e}s, and Lei]{yang2024bellman}
Zitong Yang, Emmanuel Cand{\`e}s, and Lihua Lei.
\newblock Bellman conformal inference: Calibrating prediction intervals for time series.
\newblock \emph{arXiv preprint arXiv:2402.05203}, 2024.

\bibitem[Lei and Cand{\`e}s(2021)]{Lei2021ConformalIO}
Lihua Lei and Emmanuel~J. Cand{\`e}s.
\newblock Conformal inference of counterfactuals and individual treatment effects.
\newblock \emph{Journal of the Royal Statistical Society: Series B (Statistical Methodology)}, 83, 2021.

\bibitem[Papadopoulos et~al.(2011{\natexlab{a}})Papadopoulos, Vovk, and Gammerman]{DBLP:journals/jair/PapadopoulosVG11}
Harris Papadopoulos, Vladimir Vovk, and Alexander Gammerman.
\newblock Regression conformal prediction with nearest neighbours.
\newblock \emph{J. Artif. Intell. Res.}, 40:\penalty0 815--840, 2011{\natexlab{a}}.
\newblock URL \url{http://jair.org/papers/paper3198.html}.

\bibitem[Romano et~al.(2019{\natexlab{b}})Romano, Patterson, and Cand{\`{e}}s]{DBLP:conf/nips/RomanoPC19}
Yaniv Romano, Evan Patterson, and Emmanuel~J. Cand{\`{e}}s.
\newblock Conformalized quantile regression.
\newblock In Hanna~M. Wallach, Hugo Larochelle, Alina Beygelzimer, Florence d'Alch{\'{e}}{-}Buc, Emily~B. Fox, and Roman Garnett, editors, \emph{Advances in Neural Information Processing Systems 32: Annual Conference on Neural Information Processing Systems 2019, NeurIPS 2019, December 8-14, 2019, Vancouver, BC, Canada}, pages 3538--3548, 2019{\natexlab{b}}.
\newblock URL \url{https://proceedings.neurips.cc/paper/2019/hash/5103c3584b063c431bd1268e9b5e76fb-Abstract.html}.

\bibitem[Sesia and Cand{\`e}s(2020)]{sesia2020comparison}
Matteo Sesia and Emmanuel~J Cand{\`e}s.
\newblock A comparison of some conformal quantile regression methods.
\newblock \emph{Stat}, 9\penalty0 (1):\penalty0 e261, 2020.

\bibitem[Izbicki et~al.(2020{\natexlab{a}})Izbicki, Shimizu, and Stern]{Izbicki2020DistributionfreeCP}
Rafael Izbicki, Gilson~T. Shimizu, and Rafael~Bassi Stern.
\newblock Distribution-free conditional predictive bands using density estimators.
\newblock \emph{ArXiv}, abs/1910.05575, 2020{\natexlab{a}}.

\bibitem[Sesia and Romano(2021)]{DBLP:conf/nips/SesiaR21}
Matteo Sesia and Yaniv Romano.
\newblock Conformal prediction using conditional histograms.
\newblock In Marc'Aurelio Ranzato, Alina Beygelzimer, Yann~N. Dauphin, Percy Liang, and Jennifer~Wortman Vaughan, editors, \emph{Advances in Neural Information Processing Systems 34: Annual Conference on Neural Information Processing Systems 2021, NeurIPS 2021, December 6-14, 2021, virtual}, pages 6304--6315, 2021.
\newblock URL \url{https://proceedings.neurips.cc/paper/2021/hash/31b3b31a1c2f8a370206f111127c0dbd-Abstract.html}.

\bibitem[Izbicki et~al.(2020{\natexlab{b}})Izbicki, Shimizu, and Stern]{izbicki2020cd}
Rafael Izbicki, Gilson Shimizu, and Rafael~B Stern.
\newblock Cd-split and hpd-split: efficient conformal regions in high dimensions.
\newblock \emph{arXiv preprint arXiv:2007.12778}, 2020{\natexlab{b}}.

\bibitem[Yang and Kuchibhotla(2021)]{yang2021finite}
Yachong Yang and Arun~Kumar Kuchibhotla.
\newblock Finite-sample efficient conformal prediction.
\newblock \emph{arXiv preprint arXiv:2104.13871}, 2021.

\bibitem[Han et~al.(2022)Han, Tang, Ghosh, and Liu]{han2022split}
Xing Han, Ziyang Tang, Joydeep Ghosh, and Qiang Liu.
\newblock Split localized conformal prediction.
\newblock \emph{arXiv preprint arXiv:2206.13092}, 2022.

\bibitem[Papadopoulos et~al.(2008)Papadopoulos, Gammerman, and Vovk]{papadopoulos2008normalized}
Harris Papadopoulos, Alex Gammerman, and Volodya Vovk.
\newblock Normalized nonconformity measures for regression conformal prediction.
\newblock In \emph{Proceedings of the IASTED International Conference on Artificial Intelligence and Applications (AIA 2008)}, pages 64--69, 2008.

\bibitem[Papadopoulos et~al.(2011{\natexlab{b}})Papadopoulos, Vovk, and Gammerman]{papadopoulos2011regression}
Harris Papadopoulos, Vladimir Vovk, and Alex Gammerman.
\newblock Regression conformal prediction with nearest neighbours.
\newblock \emph{Journal of Artificial Intelligence Research}, 40:\penalty0 815--840, 2011{\natexlab{b}}.

\bibitem[Papadopoulos and Haralambous(2011)]{papadopoulos2011reliable}
Harris Papadopoulos and Haris Haralambous.
\newblock Reliable prediction intervals with regression neural networks.
\newblock \emph{Neural Networks}, 24\penalty0 (8):\penalty0 842--851, 2011.

\bibitem[Seedat et~al.(2023)Seedat, Jeffares, Imrie, and van~der Schaar]{seedat2023improving}
Nabeel Seedat, Alan Jeffares, Fergus Imrie, and Mihaela van~der Schaar.
\newblock Improving adaptive conformal prediction using self-supervised learning.
\newblock In \emph{International Conference on Artificial Intelligence and Statistics}, pages 10160--10177. PMLR, 2023.

\bibitem[Seedat et~al.(2024)Seedat, Crabb{\'e}, Qian, and van~der Schaar]{seedat2024triage}
Nabeel Seedat, Jonathan Crabb{\'e}, Zhaozhi Qian, and Mihaela van~der Schaar.
\newblock Triage: Characterizing and auditing training data for improved regression.
\newblock \emph{Advances in Neural Information Processing Systems}, 36, 2024.

\bibitem[Kuleshov et~al.(2018)Kuleshov, Fenner, and Ermon]{DBLP:conf/icml/KuleshovFE18}
Volodymyr Kuleshov, Nathan Fenner, and Stefano Ermon.
\newblock Accurate uncertainties for deep learning using calibrated regression.
\newblock In Jennifer~G. Dy and Andreas Krause, editors, \emph{Proceedings of the 35th International Conference on Machine Learning, {ICML} 2018, Stockholmsm{\"{a}}ssan, Stockholm, Sweden, July 10-15, 2018}, volume~80 of \emph{Proceedings of Machine Learning Research}, pages 2801--2809. {PMLR}, 2018.
\newblock URL \url{http://proceedings.mlr.press/v80/kuleshov18a.html}.

\bibitem[Nixon et~al.(2019)Nixon, Dusenberry, Zhang, Jerfel, and Tran]{DBLP:conf/cvpr/NixonDZJT19}
Jeremy Nixon, Michael~W. Dusenberry, Linchuan Zhang, Ghassen Jerfel, and Dustin Tran.
\newblock Measuring calibration in deep learning.
\newblock In \emph{{IEEE} Conference on Computer Vision and Pattern Recognition Workshops, {CVPR} Workshops 2019, Long Beach, CA, USA, June 16-20, 2019}, pages 38--41. Computer Vision Foundation / {IEEE}, 2019.
\newblock URL \url{http://openaccess.thecvf.com/content\_CVPRW\_2019/html/Uncertainty\_and\_Robustness\_in\_Deep\_Visual\_Learning/Nixon\_Measuring\_Calibration\_in\_Deep\_Learning\_CVPRW\_2019\_paper.html}.

\bibitem[Chang et~al.(2024)Chang, Wang, Wang, Wu, Yang, Zhu, Chen, Yi, Wang, Wang, et~al.]{chang2024survey}
Yupeng Chang, Xu~Wang, Jindong Wang, Yuan Wu, Linyi Yang, Kaijie Zhu, Hao Chen, Xiaoyuan Yi, Cunxiang Wang, Yidong Wang, et~al.
\newblock A survey on evaluation of large language models.
\newblock \emph{ACM Transactions on Intelligent Systems and Technology}, 15\penalty0 (3):\penalty0 1--45, 2024.

\bibitem[Blundell et~al.(2015)Blundell, Cornebise, Kavukcuoglu, and Wierstra]{DBLP:conf/icml/BlundellCKW15}
Charles Blundell, Julien Cornebise, Koray Kavukcuoglu, and Daan Wierstra.
\newblock Weight uncertainty in neural network.
\newblock In Francis~R. Bach and David~M. Blei, editors, \emph{Proceedings of the 32nd International Conference on Machine Learning, {ICML} 2015, Lille, France, 6-11 July 2015}, volume~37 of \emph{{JMLR} Workshop and Conference Proceedings}, pages 1613--1622. JMLR.org, 2015.
\newblock URL \url{http://proceedings.mlr.press/v37/blundell15.html}.

\bibitem[Hern{\'{a}}ndez{-}Lobato and Adams(2015)]{DBLP:conf/icml/Hernandez-Lobato15b}
Jos{\'{e}}~Miguel Hern{\'{a}}ndez{-}Lobato and Ryan~P. Adams.
\newblock Probabilistic backpropagation for scalable learning of bayesian neural networks.
\newblock In Francis~R. Bach and David~M. Blei, editors, \emph{Proceedings of the 32nd International Conference on Machine Learning, {ICML} 2015, Lille, France, 6-11 July 2015}, volume~37 of \emph{{JMLR} Workshop and Conference Proceedings}, pages 1861--1869. JMLR.org, 2015.
\newblock URL \url{http://proceedings.mlr.press/v37/hernandez-lobatoc15.html}.

\bibitem[Li and Gal(2017)]{DBLP:conf/icml/LiG17}
Yingzhen Li and Yarin Gal.
\newblock Dropout inference in bayesian neural networks with alpha-divergences.
\newblock In Doina Precup and Yee~Whye Teh, editors, \emph{Proceedings of the 34th International Conference on Machine Learning, {ICML} 2017, Sydney, NSW, Australia, 6-11 August 2017}, volume~70 of \emph{Proceedings of Machine Learning Research}, pages 2052--2061. {PMLR}, 2017.
\newblock URL \url{http://proceedings.mlr.press/v70/li17a.html}.

\bibitem[Izmailov et~al.(2021)Izmailov, Vikram, Hoffman, and Wilson]{izmailov2021bayesian}
Pavel Izmailov, Sharad Vikram, Matthew~D Hoffman, and Andrew Gordon~Gordon Wilson.
\newblock What are bayesian neural network posteriors really like?
\newblock In \emph{International conference on machine learning}, pages 4629--4640. PMLR, 2021.

\bibitem[Jospin et~al.(2022)Jospin, Laga, Boussaid, Buntine, and Bennamoun]{jospin2022hands}
Laurent~Valentin Jospin, Hamid Laga, Farid Boussaid, Wray Buntine, and Mohammed Bennamoun.
\newblock Hands-on bayesian neural networks—a tutorial for deep learning users.
\newblock \emph{IEEE Computational Intelligence Magazine}, 17\penalty0 (2):\penalty0 29--48, 2022.

\bibitem[Asuncion(2007)]{Asuncion2007UCIML}
Arthur~U. Asuncion.
\newblock Uci machine learning repository, university of california, irvine, school of information and computer sciences.
\newblock 2007.

\bibitem[Cohen et~al.(2009)Cohen, Cohen, and Banthin]{Cohen2009TheME}
Joel~W. Cohen, Steven~B. Cohen, and Jessica~S. Banthin.
\newblock The medical expenditure panel survey: A national information resource to support healthcare cost research and inform policy and practice.
\newblock \emph{Medical Care}, 47:\penalty0 S44--S50, 2009.

\bibitem[Achilles et~al.(2008)Achilles, Bain, Bellott, Boyd-Zaharias, Finn, Folger, Johnston, and Word]{achilles2008tennessee}
CM~Achilles, Helen~Pate Bain, Fred Bellott, Jayne Boyd-Zaharias, Jeremy Finn, John Folger, John Johnston, and Elizabeth Word.
\newblock Tennessee’s student teacher achievement ratio (star) project.
\newblock \emph{Harvard Dataverse}, 1:\penalty0 2008, 2008.

\bibitem[Buza(2014)]{buza2014feedback}
Krisztian Buza.
\newblock Feedback prediction for blogs.
\newblock In \emph{Data analysis, machine learning and knowledge discovery}, pages 145--152. Springer, 2014.

\bibitem[Cordts et~al.(2016)Cordts, Omran, Ramos, Rehfeld, Enzweiler, Benenson, Franke, Roth, and Schiele]{Cordts2016Cityscapes}
Marius Cordts, Mohamed Omran, Sebastian Ramos, Timo Rehfeld, Markus Enzweiler, Rodrigo Benenson, Uwe Franke, Stefan Roth, and Bernt Schiele.
\newblock The cityscapes dataset for semantic urban scene understanding.
\newblock In \emph{Proc. of the IEEE Conference on Computer Vision and Pattern Recognition (CVPR)}, 2016.

\bibitem[Tibshirani et~al.(2019{\natexlab{b}})Tibshirani, Foygel~Barber, Candes, and Ramdas]{tibshirani2019conformal}
Ryan~J Tibshirani, Rina Foygel~Barber, Emmanuel Candes, and Aaditya Ramdas.
\newblock Conformal prediction under covariate shift.
\newblock \emph{Advances in neural information processing systems}, 32, 2019{\natexlab{b}}.

\end{thebibliography}
\bibliographystyle{unsrtnat}
\clearpage
\newpage
\appendix
\begin{center}
\huge{Appendix}
\end{center}

The complete proofs are presented in Section~\ref{appendix: proofs}, and the experiment details are outlined in Section~\ref{appendix: Experimental Details total}.

\section{Theoretical Proofs}
\label{appendix: proofs}
We prove the theoretical guarantee for FFCP concerning coverage (effectiveness) in Section~\ref{appendix: ffcp_effective} and band length (efficiency) in Section~\ref{appendix: ffcp_efficient}.

\subsection{Proofs of Theorem~\ref{thm: ffcp effective}}
\label{appendix: ffcp_effective}
The proof is based on the exchangeability of data (Assumption~\ref{assu: exchangeability}) on the calibration fold and test fold, hence the key step we need to derive is the exchangeability of the non-conformity scores $\s_{\text{ff}}(\X, \Y, g \circ h) = |\Y -  f(\X)|/\| \grad g(\hat{v})\|.$
We define the relevant symbols: $\cD_{\text{tra}}$ represents the train fold, $\cD_{\text{tes}}$ represents the test fold, $\cD_{\text{cal}}$ represents the calibration fold, and $\cD'=\{(X_i,Y_i)\}_{i\in [m]}$ is the intersection of the two folds. $m$ is the number of data points in $\cD'$.

 Similar to \cite{teng2022predictive}, we first prove that for any function $\tilde{h}: \cX\times\cY \to \bR$, which is independent of $\cD^\prime$, $\tilde{h}(X_i,Y_i)$ satisfies exchangeability. 
 For the CDF $F_R$ of $\tilde{h}$ and its perturbation CDF $F^\pi_R$, $\pi$ is a random perturbation. We can conclude,

\begin{equation}
\begin{split}
     &F_R (u_1, \dots, u_{n}\ | \ \cD_{\text{tra}}) \\
     =& \bP(\tilde{h}(\X_1, \Y_1)\leq u_1, \dots, \tilde{h}(\X_n, \Y_n) \leq u_n \ | \ \cD_{\text{tra}} ) \\
     =& \bP((X_1, Y_1) \in \cC_{\tilde{h}^{-1}}(u_1-), \dots, (X_n, Y_n) \in \cC_{\tilde{h}^{-1}}(u_n-) \ | \ \cD_{\text{tra}} )\\
    =& \bP((X_{\pi(1)}, Y_{\pi(1)}) \in \cC_{\tilde{h}^{-1}}(u_1-), \dots, (X_{\pi(n)}, Y_{\pi(n)}) \in \cC_{\tilde{h}^{-1}}(u_n-) \ | \ \cD_{\text{tra}} )\\
    =& \bP(\tilde{h}(X_{\pi(1)}, Y_{\pi(1)})\leq u_1, \dots, \tilde{h}(X_{\pi(n)}, Y_{\pi(n)}) \leq u_n \ | \ \cD_{\text{tra}} ) \\
    =& F_R^\pi (u_1, \dots, u_{n}\ | \ \cD_{\text{tra}}),
\end{split}
\end{equation}
where $\cC_{\tilde{h}^{-1}}(u-) = \{(\X, \Y) : \tilde{h}(\X, \Y) \leq u\}$.

Next, we need to show the non-conformity score function
\begin{equation}
    \s_{\text{ff}}(\X, \Y, g \circ h) = |\Y -  f(\X)|/\| \grad g(\hat{v})\|,
\end{equation}
which is independent of the dataset $\cD^\prime$. 

We can see that the non-conformity score $\s_{\text{ff}}(\X, \Y, g \circ h)$ on $\cD^\prime$ uses information from $g$ and $h$, both of which depend only on the training set $\cD_{\text{tra}}$. Moreover, calculating this non-conformity score in the Algorithm~\ref{alg: ffcp} uses only single-point information, not the entire dataset $\cD^\prime$.

By integrating the aforementioned, we deduce that the non-conformity scores $\s_{\text{ff}}(\X, \Y, g \circ h)$ on $\cD^\prime$ exhibit exchangeability. 
This exchangeability, as per Lemma 1 in ~\citet{tibshirani2019conformal}, lends theoretical support to the efficacy of FFCP.

\subsection{Proofs of Theorem~\ref{thm: ffcp efficient}}
\label{appendix: ffcp_efficient}
Our main conclusions are inspired by Theorem 4 in~\cite{teng2022predictive}. The details are as follows

\textbf{Definitions.} Let $\cP$ denote the overall population distribution. 
The calibration set $\cD_{\text{cal}}$ consists of $n$ samples drawn from $\cP$. 
We denote the specific distribution of these samples as $\cP^n$.
The model under consideration, $f=g \circ h$, includes $h$ as the feature extractor and $g$ as the prediction head, with $g$ assumed to be a continuous function. 
$V^o_\cD$ represent the individual length in output space, given data set $\cD$.
The term $Q_{1-\alpha}(R)$ represents the $(1-\alpha)$-quantile of the set $R$, which adjusted to include the value ${0}$. Furthermore, $\bM[\cdot]$ signifies the mean value of a set, and subtracting a real number from a set indicates that the subtraction is applied uniformly to all elements within the set.

\emph{Vanilla CP.}\phantom{a}
Let $V^o_{\cD_{\text{cal}}} = \{v_i^o\}_{i \in \cI_{\text{cal}}}$ denote the {individual} length in the output space for vanilla CP, given the calibration set $\cD_{\text{cal}}$.
Since vanilla CP returns band length with $1-\alpha$ quantile of non-conformity score, the resulting average band length is derived by $2Q_{1-\alpha} (V^o_{\cD_{\text{cal}}})$. 

\emph{Fast Feature CP.}\phantom{a}
According to the definition of FFCP, $V_\cD^f=V_\cD^o/\|\grad g(\hat{v})\|$, 

The resulting band length in FFCP is denoted by $2\bE_{(X^\prime, Y^\prime)\sim \cP} (\|\grad g(\hat{v^\prime})\|\cdot Q_{1-\alpha}(V^o_{\cD_{\text{cal}}}/\|\grad g(\hat{v}_{\text{cal}})\|)$.

\begin{theorem}
\label{thm: ffcp efficient}
(FFCP is provably more efficient). Assume that the non-conformity score is in norm-type. We assume a Holder assumption that there exist $\alpha>0, L>0$ such that $| \cH(x;X) - \cH(y;X) | \leq L |x - y|^\alpha$ for all $X$, where $\cH$ is any function. Then the feature space satisfies the following square conditions:

\vspace{-0.05cm}
\begin{enumerate}
    \item \textbf{Expansion.} The feature space expands the differences between individual length and their quantiles, namely, $L \bE_{\cD \sim \cP^n} \bM | Q_{1-\alpha}(V_\cD^o/\|\grad g(\hat{v})\|) - V_\cD^o/\|\grad g(\hat{v})\| |^\alpha < \bE_{\cD \sim \cP^n}
    \bM [Q_{1-\alpha}( V_\cD^o ) - V_\cD^o] - 2 \max\{L, 1\} (c/\sqrt{n})^{\min\{\alpha, 1\}}$.
    
    \item \textbf{Quantile Stability.} Given a calibration set $\cD_{\text{cal}}$, the quantile of the band length is stable in both feature space and output space, namely, $\bE_{\cD \sim \cP^n} | Q_{1-\alpha}(V_\cD^o/\|\grad g(\hat{v})\|) - Q_{1-\alpha}(V_{{\cD_{\text{cal}}}}^o/\|g(\grad \hat{v}_{\text{cal}})\|) |  \leq \frac{c}{\sqrt{n}}$ and $ \bE_{\cD \sim \cP^n} | Q_{1-\alpha}(V_{{\cD}}^o) - Q_{1-\alpha}(V_{{\cD_{\text{cal}}}}^o) |  \leq \frac{c}{\sqrt{n}}$.

    \vspace{-0.05cm}
\end{enumerate}

Then FFCP provably outperforms vanilla CP in terms of average band length, namely,
\[
\bE_{(X^\prime, Y^\prime)\sim \cP} (\|\grad g(\hat{v^\prime})\|\cdot Q_{1-\alpha}(V^o_{\cD_{\text{cal}}}/\|\grad g(\hat{v}_{\text{cal}})\|) < Q_{1-\alpha}(V_{D_\text{cal}}^0),
\]
where the expectation is taken over the calibration fold and the testing point $(X', Y')$.
\end{theorem}

\begin{proof}[Proof of Theorem~\ref{thm: ffcp efficient}]
We first proof with \emph{Expansion} Assumption,
\begin{equation}
\begin{split}
    &L \bE_{\cD \sim \cP^n} \bM | Q_{1-\alpha}(V_\cD^o/\|\grad g(\hat{v})\|) - V_\cD^o/\|\grad g(\hat{v})\| |^\alpha < 
    \bE_{\cD \sim \cP^n}
    \bM [Q_{1-\alpha}( V_\cD^o ) - V_\cD^o] \\ 
    & - 2 \max\{L, 1\} (c/\sqrt{n})^{\min\{\alpha, 1\}}.
\end{split}
\end{equation}

And we can obtain
\begin{equation}
\begin{split}
    \bE_{\cD} \bM V^o_\cD <&  \bE_{\cD} Q_{1-\alpha}( V_\cD^o )  \\
    &- 2\max\{L, 1\} (c/\sqrt{n})^{\min\{\alpha, 1\}} - L \bE_{\cD \sim \cP^n} \bM | Q_{1-\alpha}(V_\cD^o/\|\grad g(\hat{v})\|) - V_\cD^o/\|\grad g(\hat{v})\| |^\alpha .
\end{split}
\end{equation}

According to Holder condition for quantile function, we obtain that $ \bM (\|\grad g(\hat{v})\|\cdot Q_{1-\alpha}(V_\cD^o/\|\grad g(\hat{v})\|)) \\
\le \bM  V^o_\cD + L\bM | Q_{1-\alpha}(V_\cD^o/\|\grad g(\hat{v})\|) - V_\cD^o/\|\grad g(\hat{v})\| |^\alpha$, therefore
\begin{equation}
    \bE_{\cD} \bM (\|\grad g(\hat{v})\|\cdot Q_{1-\alpha}(V_\cD^o/\|\grad g(\hat{v})\|)) < \bE_{\cD} Q_{1-\alpha}(V_{\cD}^o)  - 2 \max\{1, L\} [c/\sqrt{n}]^{\min\{1, \alpha\}}.
\end{equation}

As the \emph{Quantile Stability} assumption, we have that $\bE_{\cD \sim \cP^n} | Q_{1-\alpha}(V_\cD^o/\|\grad g(\hat{v})\|) - Q_{1-\alpha}(V_{{\cD_{\text{cal}}}}^o/\|\grad g(\hat{v}_{\text{cal}})\|) |  \\
\leq \frac{c}{\sqrt{n}}$ and $ \bE_{\cD \sim \cP^n} | Q_{1-\alpha}(V_{{\cD}}^o) - Q_{1-\alpha}(V_{{\cD_{\text{cal}}}}^o) |  \leq \frac{c}{\sqrt{n}}$.
Therefore, 
\begin{equation}
\begin{split}
&2\bE \bm (\|\grad g(\hat{v})\|\cdot Q_{1-\alpha}(V^o_{\cD_{\text{cal}}}/\|\grad g(\hat{v}_{\text{cal}})\|)  \\
<& 2 Q_{1-\alpha}(V_{\cD}^o)  - 2 \max\{1, L\} [c/\sqrt{n}]^{\min\{1, \alpha\}}. \\
    <& 2 Q_{1-\alpha}(V_{\cD}^o) .
\end{split}
\end{equation}

\end{proof}

\section{Experimental Details}
\label{appendix: Experimental Details total}
Section~\ref{appendix: Experimental Details} introduces the omitted experimental details. Section~\ref{appendix: Certifying Square Conditions} certifies the square conditions. Section~\ref{appendix: Robustness} discusses discusses the robustness of FFCP coverage with respect to the splitting point and across each network layer.
Section~\ref{appendix: fcp advantage untrained not work} demonstrates that FFCP performs similarly to vanilla CP in untrained neural networks, confirming that FFCP's efficiency is due to the semantic information trained in the feature space. 
Section~\ref{appendix: FFCQR} proposes FFCQR after applying the gradient-level techniques of FFCP to CQR. 
Section~\ref{appendix: FFLCP} proposes FFLCP after applying the gradient-level techniques of FFCP to LCP. 
Section~\ref{appendix: FFRAPS} proposes FFRAPS after applying the gradient-level techniques of FFCP to RAPS. 
Finally, Section~\ref{appendix:Additional-Experiment-Results} provides additional experimental results.

\subsection{Experimental Details}
\label{appendix: Experimental Details}

\textbf{Model Architecture.}
For the one-dimensional we employ a four-layer neural network, with each layer consisting of 64 dimensions. 
For the semantic segmentation experiment, we utilize a network architecture combining ResNet50 with two additional convolutional layers. 
We use ResNet50 as the base feature extractor $h$,  and the two subsequent convolution layers form the prediction head $g$.

\subsection{Verifying Square Conditions}
\label{appendix: Certifying Square Conditions}
We verify the square conditions in this section. 
The key component of the square conditions is \emph{Expansion} condition, which states that performing the quantile step does not result in a significant loss of efficiency.

For computational simplicity, We take exponent 
 $\alpha = 1$ and do not consider the Lipschitz factor $L$.
We next provide experiment results in Figure~\ref{fig:VCP-FFCP-score-Histogram} on comparing the distribution of the scores between Vanilla CP with FFCP.

From the figure, we observe that the overall distribution of FFCP non-conformity scores is closer to the quantile. This numerically validates that $\mathbf{M} \left| Q_{1-\alpha}(V_{\mathcal{D}}^o/|\nabla g(\hat{v})|) - V_{\mathcal{D}}^o/|\nabla g(\hat{v})|\right|$ is less than $\mathbf{M} \left[ Q_{1-\alpha}( V_{\mathcal{D}}^o ) - V_{\mathcal{D}}^o \right]$.

\begin{figure}[t]
    \centering
    \subfigure[Vanilla CP]{\includegraphics[height=0.3\linewidth]{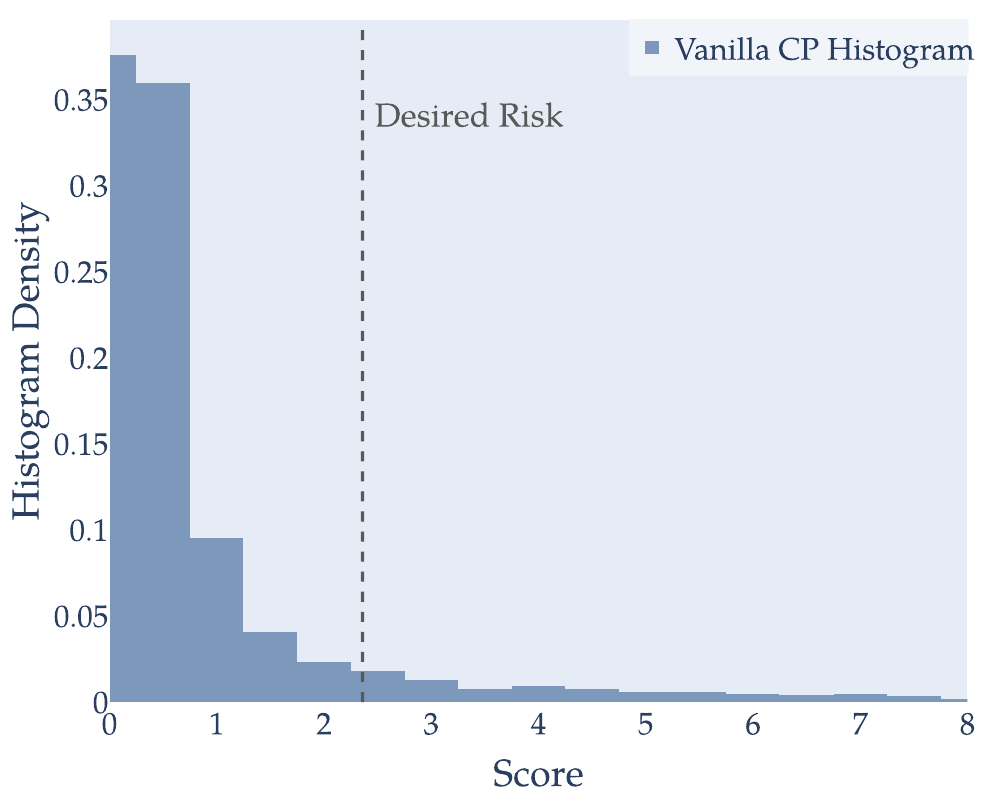}}
    \subfigure[FFCP]{\includegraphics[height=0.3\linewidth]{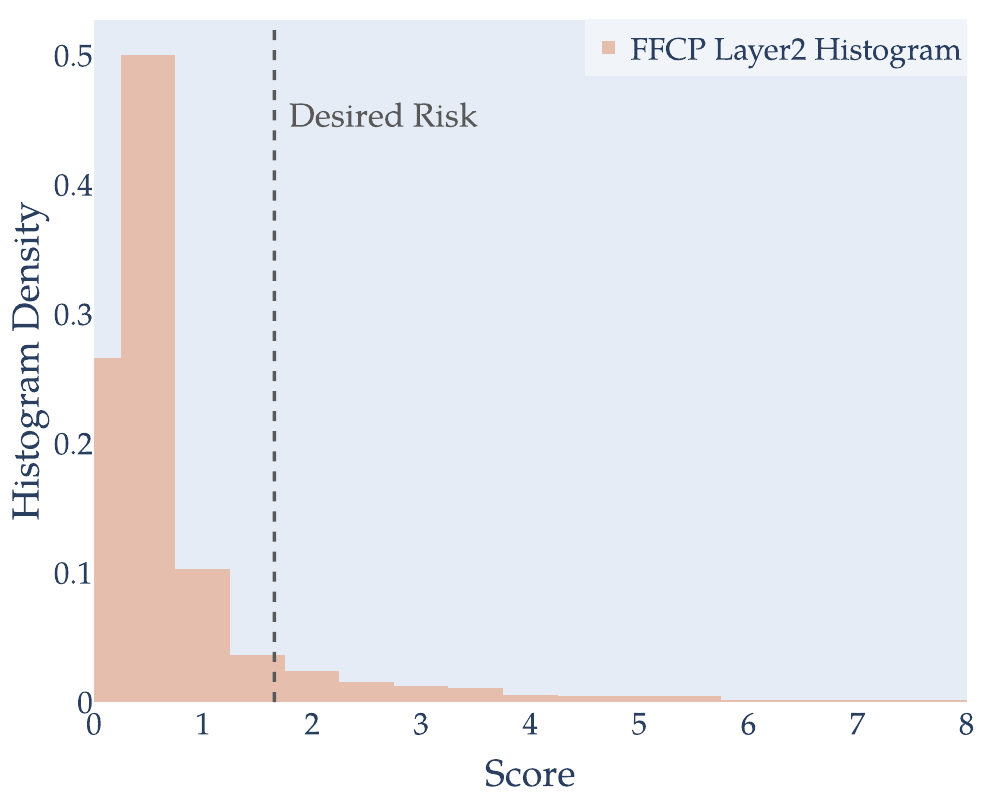}}
    \vspace{0cm}
    \caption{Empirical validation of Theorem~\ref{thm: ffcp efficient}. 
    We plot the score distributions and their corresponding quantiles ($\alpha=0.1$) of Vanilla CP (left) and FFCP (right).
    Compared to Vanilla CP, the non-conformity scores of FFCP are closer to their quantiles, leading to a shorter band.
    Compared to Vanilla, FFCP exhibits a more stable distribution with higher quantiles, leading to better performance for FFCP. FFCP selects layer 2 for display.}
    \label{fig:VCP-FFCP-score-Histogram}
\end{figure}

\subsection{Robustness of FFCP}
\label{appendix: Robustness}
To verify that the coverage by FFCP maintains its robustness despite changes in the splitting point, we performed a network split. 
The experimental results, detailed in Table~\ref{tab:layer-ablation}, demonstrate that FFCP is indeed robust.
\begin{table}[h]
\caption{Ablation study of the number of layers in $h$ and $g$ in unidimensional tasks. For the sake of avoiding redundancy, we set 
$\alpha=0.05$.
}
\label{tab:layer-ablation}
\begin{center}
\begin{scriptsize}
\begin{sc}
\resizebox{\linewidth}{!}{
\begin{tabular}{cccccccc}
\toprule
\multicolumn{2}{c}{Dataset} &  \multicolumn{2}{c}{facebook1}  &  \multicolumn{2}{c}{meps19}  &  \multicolumn{2}{c}{blog}\\
\cmidrule(lr){3-4} \cmidrule(lr){5-6} \cmidrule(lr){7-8}
\multicolumn{1}{c}{Method} & \multicolumn{1}{c}{Number($\hat{g} \circ \hat{h}$)} & Coverage  &  Length  &  Coverage  &  Length  &  Coverage  &  Length \\
\midrule
\midrule
Vanilla CP & $/$  &  $95.24\pm0.16$  & 
 $4.60\pm0.50$  &  $95.35\pm0.23$  &  $7.34\pm1.01$  &  $95.08\pm0.11$  &  $7.88\pm0.97$  \\ 
\cmidrule(lr){1-8}
\multirow{5}{*}{FFCP}  
&  $h:0 \quad g:4$   &  $94.88\pm0.19$  &  $5.35\pm0.42$  &  $95.18\pm0.40$  &  $5.36\pm0.52$  &  $94.97\pm0.31$  &  $8.28\pm0.59$  \\ 
{}    &   $h:1 \quad g:3$   &  $94.84\pm0.14$  &  $4.21\pm0.46$  &  $95.13\pm0.38$  &  $5.15\pm0.49$  &  $94.99\pm0.18$  &  $6.37\pm1.22$  \\ 
{}    &   $h:2 \quad g:2$   &  $95.14\pm0.13$  &  $2.48\pm0.09$  &  $95.19\pm0.36$  &  $5.57\pm0.87$  &  $95.08\pm0.12$  &  $5.36\pm0.93$  \\ 
{}    &   $h:3 \quad g:1$   &  $95.16\pm0.18$  &  $2.59\pm0.72$  &  $95.34\pm0.22$  &  $6.95\pm1.13$  &  $95.05\pm0.11$  &  $6.46\pm1.52$  \\ 
{}    &   $h:4 \quad g:0$   &  $95.24\pm0.16$  &  $4.60\pm0.50$  &  $95.35\pm0.23$  &  $7.34\pm1.01$  &  $95.08\pm0.11$  &  $7.88\pm0.97$ \\
\bottomrule
\end{tabular}
}
\end{sc}
\end{scriptsize}
\end{center}
\end{table}

\subsection{FFCP works due to semantic information in feature space}
\label{appendix: fcp advantage untrained not work}
One of our primary advantages is that FFCP leverages the semantic information of gradient in feature space. 
This is due to the fact that gradient-level techniques in feature space improve efficiency via the robust feature embedding abilities of well-trained neural networks.

On the other hand, when the base model is untrained and initialized randomly, lacking meaningful semantic representation in gradient, the band length produced by FFCP is comparable to Vanilla CP. 
For results, see Table~\ref{tab1:untrain-model}.

\subsection{FFCQR}
\label{appendix: FFCQR}
This section highlights the adaptability of FFCP's gradient-level techniques, showing their suitability for a wide range of existing conformal prediction algorithms.
We choose Conformalized Quantile Regression (CQR, \citet{DBLP:conf/nips/RomanoPC19}) to propose Fast Feature Conformalized Quantile Regression (FFCQR).
The fundamental concept is similar to FFCP Algorithm~\ref{alg: ffcp}, where calibration steps are performed within the gradient information.
FFCQR algorithm is proposed in Algorithm~\ref{alg: ffcqr}.
\begin{algorithm*}[h]
\caption{Fast Feature Conformalized Quantile Regression (FFCQR)} 
\label{alg: ffcqr} 
\begin{algorithmic}[1] 
\REQUIRE Confidence level $\alpha$, 
dataset $\cD = \{ (\Xs, \Ys)\}_{i \in \cI}$, test point $\X^\prime$;

\STATE{Randomly split the dataset $\mathcal{D}$ into a training fold $\cD_{\text{tra}} \triangleq (\X_i, \Y_i)_{i \in \mathcal{I}_{\text{tra}}}$ together with a calibration fold $\cD_{\text{cal}} \triangleq (\X_i, \Y_i)_{i \in \mathcal{I}_{\text{cal}}}$;}

\STATE{Train a base machine learning model $f^{\text{lo}}=g^{\text{lo}}\circ h(\cdot)$ and $f^{\text{hi}}=g^{\text{hi}}\circ h(\cdot)$ using $\cD_{\text{tra}}$ to estimate the quantile of response $Y_i$, which returns $[f^{\text{lo}}(\X_i), f^{\text{hi}}(\X_i)]$;}

\STATE{For each $i \in \mathcal{I}_{\text{cal}}$, calculate the non-conformity score $\tilde{R}_i^{\text{lo}} = (f^{\text{lo}}(\X_i) - Y_i)/\| \grad {g}^{\text{lo}}(\hat{v}_i)\|$ and $\tilde{R}_i^{\text{hi}} = (Y_i - f^{\text{hi}}(\X_i))/\| \grad {g}^{\text{hi}}(\hat{v}_i)\|,$ where $\grad {g}(\cdot)$ denote the gradient of ${g}(\cdot)$ on the feature $\hat{v}_i\triangleq h(\X_i)$, namely $\grad {g}^{\text{lo}}(\hat{v}_i) = \frac{\mathrm{d} {g^{\text{lo}} \circ {h}(\X_i)}}{\mathrm{d} {{h}(\X_i)}}$ and $\grad {g}^{\text{hi}}(\hat{v}_i) = \frac{\mathrm{d} {g^{\text{hi}} \circ {h}(\X_i)}}{\mathrm{d} {{h}(\X_i)}}$ }

\STATE{Calculate the $(1-\alpha)$-th quantile $Q_{1-\alpha}$ of the distribution $\frac{1}{|\cI_{\text{cal}}| + 1} \sum_{i \in \mathcal{I}_{\text{cal}}}  \delta_{\tilde{R}_i} + {\delta}_{\infty}$, where $\tilde{R}_i=\max\left\{\tilde{R}_i^{\text{lo}},\tilde{R}_i^{\text{hi}}\right\}$}

\ENSURE $\cC_{1-\alpha}^{\text{ffcqr}}(\X^\prime) =  \left[ f^{\text{lo}}(\X^\prime) - \|\nabla{g}^{\text{lo}}(\hat{v}^\prime)\| \cdot Q_{1-\alpha}, f^{\text{hi}}(\X^\prime) + \|\nabla {g}^{\text{hi}}(\hat{v}^\prime)\|\cdot Q_{1-\alpha}\right]$, where $\hat{v}^\prime = h(\X^\prime)$.
\end{algorithmic} 
\end{algorithm*}

We summarize run time in Table~\ref{tab:FFCQR-time-compare} and the experiments result in Table~\ref{tab:ffcqr-meps19} (meps19), Table~\ref{tab:ffcqr-com} (com), and Table~\ref{tab:ffcqr-bike} (bike). 
FFCQR reduces runtime compared to FCQR, while achieving better efficiency compared to CQR.

Furthermore, we have observed that as the values of $[\alpha, 1-\alpha]$ used by the neural networks in all CQR methods (CQR, FCQR and FFCQR) become increasingly closer in the training process (The level difference between [0.1, 0.9] is 0.8, while the level difference between [0.49, 0.51] is 0.02, with the difference gradually decreasing), the band length returned by FFCQR gradually narrows. 
This implies that our method holds an advantage on returned band length when the narrower neural network  confidence level.

\begin{table*}[h]
\caption{{Time Comparison among CQR, FCQR and FFCQR. For quantile regression tasks, FFCQR also demonstrates more efficient performance. The last column represents the speed improvement factor of FFCQR compared to FCQR. The time unit is in seconds.}
}
\label{tab:FFCQR-time-compare}
\begin{center}
\begin{small}
\begin{sc}
\begin{tabular}{ccccc}
\toprule
Dataset  &  \multicolumn{1}{c}{CQR}  &  \multicolumn{1}{c}{FCQR}
&  \multicolumn{1}{c}{FFCQR}     & \multicolumn{1}{c}{FASTER} \\
\midrule
\midrule
synthetic  
& $0.0125$\scriptsize{$\pm 0.0062$} 
& $0.3237$\scriptsize{$\pm 0.0152$} 
& $0.0742$\scriptsize{$\pm 0.0091$} & 4x  \\
com  
& $0.0045$\scriptsize{$\pm 0.0015$} 
& $0.2730$\scriptsize{$\pm 0.1088$} 
& $0.0210$\scriptsize{$\pm 0.0011$} & 13x \\
fb1 
& $0.0446$\scriptsize{$\pm 0.0157$}
& $1.7276$\scriptsize{$\pm 0.1389$}&  
$0.2532$\scriptsize{$\pm 0.0166$} & 7x\\
fb2 
& $0.0812$\scriptsize{$\pm 0.0187$}
& $3.9967$\scriptsize{$\pm 0.7330$}&  
$0.0617$\scriptsize{$\pm 0.0123$} & 65x\\
meps19  
& $0.0187$\scriptsize{$\pm 0.0018$} 
& ${0.7671}$\scriptsize{$\pm 0.0438$}
& $0.1189$\scriptsize{$\pm 0.0048$} & 6x\\
meps20  
& $0.0438$\scriptsize{$\pm 0.0079$} 
& $1.1876$\scriptsize{$\pm 0.2206$}  
& $0.1505$\scriptsize{$\pm 0.0138$}  & 8x\\
meps21 
& $0.0187$\scriptsize{$\pm 0.0027$} 
& $0.8004$\scriptsize{$\pm 0.0657$}
& $0.1120$\scriptsize{$\pm 0.0053$} & 7x \\
star 
& $0.0047$\scriptsize{$\pm 0.0009$} 
& $0.2352$\scriptsize{$\pm 0.0419$} 
& $0.0214$\scriptsize{$\pm 0.0005$} & 11x\\
bio 
& $0.0774$\scriptsize{$\pm 0.0541$}
& $6.9365$\scriptsize{$\pm 4.5494$}
& $0.6473$\scriptsize{$\pm 0.4879$}& 11x\\
blog 
& $0.1121$\scriptsize{$\pm 0.0153$}
& $1.9591$\scriptsize{$\pm 0.1346$}&  
$0.3941$\scriptsize{$\pm 0.0618$} & 5x\\
bike 
& $0.0138$\scriptsize{$\pm 0.0045$}
& $ 1.8528$\scriptsize{$\pm 2.3969$}
& $0.2382$\scriptsize{$\pm 0.3261$}& 6x\\
\bottomrule
\end{tabular}
\end{sc}
\end{small}
\end{center}
\end{table*}

\begin{table*}[h]
\caption{Coverage and Band Length at Different Confidence Levels Used By the Neural Networks in CQR methods with \emph{Meps19} detaset. FFCQR yields shorter band lengths compared to CQR.}
\label{tab:ffcqr-meps19}
\begin{center}
\begin{small}
\begin{sc}
\resizebox{\linewidth}{!}{
\begin{tabular}{cccccccccccc}
\toprule
\multicolumn{2}{c}{Confidence Levels used by NN}  
&  \multicolumn{2}{c}{[0.1, 0.9]}  &  \multicolumn{2}{c}{[0.2, 0.8]} &  \multicolumn{2}{c}{[0.3, 0.7]}&  \multicolumn{2}{c}{[0.4, 0.6]}
&  \multicolumn{2}{c}{[0.49, 0.51]}\\
\cmidrule(lr){1-2}  \cmidrule(lr){3-4}  \cmidrule(lr){5-6}  \cmidrule(lr){7-8}  \cmidrule(lr){9-10}  \cmidrule(lr){11-12}  
\multicolumn{2}{c}{Metrics}  &  Coverage  &  Length  
&  Coverage  &  Length  &  Coverage  &  Length  
&  Coverage  &  Length  &  Coverage  &  Length  \\
\midrule
\midrule
\multicolumn{2}{c}{Vanilla CQR}  
&  $90.28$\scriptsize{$\pm 0.47$}  &  $2.43$\scriptsize{$\pm 0.11$}  
&  $90.19$\scriptsize{$\pm 0.46$}  &  $2.58$\scriptsize{$\pm 0.45$}  
&  $90.63$\scriptsize{$\pm 0.32$}  &  $2.93$\scriptsize{$\pm 0.51$}  
&  $90.48$\scriptsize{$\pm 0.42$}  &  $3.47$\scriptsize{$\pm 0.16$}  
&  $90.44$\scriptsize{$\pm 0.36$}  &  $3.48$\scriptsize{$\pm 0.08$}  \\
\multicolumn{2}{c}{FCQR}  
&  $91.32$\scriptsize{$\pm 0.37$}  &  $1.50$\scriptsize{$\pm 0.37$}  
&  $90.26$\scriptsize{$\pm 0.33$}  &  $2.61$\scriptsize{$\pm 2.01$}
&  $90.45$\scriptsize{$\pm 0.54$}  &  $2.30$\scriptsize{$\pm 2.38$}
&  $90.58$\scriptsize{$\pm 0.33$}  &  $6.11$\scriptsize{$\pm 0.89$}  
&  $90.47$\scriptsize{$\pm 0.45$}  &  $4.59$\scriptsize{$\pm 1.73$}  \\
\multirow{5}{*}{FFCQR}  &  layer0  
&  $90.29$\scriptsize{$\pm 0.60$}  &  $2.61$\scriptsize{$\pm 0.13$}  
&  $90.22$\scriptsize{$\pm 0.26$}  &  $2.58$\scriptsize{$\pm 0.54$}  
&  $90.31$\scriptsize{$\pm 0.43$}  &  $3.03$\scriptsize{$\pm 0.90$}
&  $89.95$\scriptsize{$\pm 0.29$}  &  $5.30$\scriptsize{$\pm 0.57$}  
&  $89.84$\scriptsize{$\pm 0.12$}  &  $5.96$\scriptsize{$\pm 1.26$}  \\
{}  &  layer1  
&  $90.29$\scriptsize{$\pm 0.57$}  &  $2.56$\scriptsize{$\pm 0.13$}  
&  $90.10$\scriptsize{$\pm 0.33$}  &  $2.49$\scriptsize{$\pm 0.52$}  
&  $90.34$\scriptsize{$\pm 0.46$}  &  $2.92$\scriptsize{$\pm 0.85$}  
&  $90.02$\scriptsize{$\pm 0.33$}  &  $5.07$\scriptsize{$\pm 0.50$}  
&  $89.88$\scriptsize{$\pm 0.24$}  &  $5.71$\scriptsize{$\pm 1.22$}  \\
{}  &  layer2  
&  $90.14$\scriptsize{$\pm 0.60$}  &  $2.34$\scriptsize{$\pm 0.12$}  
&  $90.24$\scriptsize{$\pm 0.65$}  &  $2.22$\scriptsize{$\pm 0.32$}
&  $90.34$\scriptsize{$\pm 0.43$}  &  $2.60$\scriptsize{$\pm 0.68$}
&  $89.96$\scriptsize{$\pm 0.40$}  &  $4.10$\scriptsize{$\pm 0.22$}  
&  $89.97$\scriptsize{$\pm 0.33$}  &  $4.76$\scriptsize{$\pm 1.17$}  \\
{}  &  layer3  
&  $90.21$\scriptsize{$\pm 0.45$}  &  $2.18$\scriptsize{$\pm 0.12$}  
&  $90.14$\scriptsize{$\pm 0.42$}  &  $2.10$\scriptsize{$\pm 0.19$}  
&  $90.35$\scriptsize{$\pm 0.36$}  &  $2.34$\scriptsize{$\pm 0.19$}  
&  $90.28$\scriptsize{$\pm 0.33$}  &  $2.76$\scriptsize{$\pm 0.15$}  
&  $89.86$\scriptsize{$\pm 0.49$}  &  $3.16$\scriptsize{$\pm 0.62$}  \\
{}  &  layer4  
&  $90.28$\scriptsize{$\pm 0.47$}  &  $2.43$\scriptsize{$\pm 0.11$}  
&  $90.19$\scriptsize{$\pm 0.46$}  &  $2.58$\scriptsize{$\pm 0.45$}  
&  $90.63$\scriptsize{$\pm 0.32$}  &  $2.93$\scriptsize{$\pm 0.51$}  
&  $90.48$\scriptsize{$\pm 0.42$}  &  $3.47$\scriptsize{$\pm 0.16$}  
&  $90.44$\scriptsize{$\pm 0.36$}  &  $3.48$\scriptsize{$\pm 0.08$}  \\
\bottomrule
\end{tabular}
}
\end{sc}
\end{small}
\end{center}
\end{table*}

\begin{table*}[h]
\caption{Coverage and Band Length at Different Confidence Levels Used By the Neural Networks in CQR methods with \emph{com} dataset. FFCQR yields shorter band lengths compared to CQR.}
\label{tab:ffcqr-com}
\begin{center}
\begin{small}
\begin{sc}
\resizebox{\linewidth}{!}{
\begin{tabular}{cccccccccccc}
\toprule
\multicolumn{2}{c}{Confidence Levels Uesd by NN}  &  \multicolumn{2}{c}{[0.1, 0.9]}  &  \multicolumn{2}{c}{[0.2, 0.8]} &  \multicolumn{2}{c}{[0.3, 0.7]}&  \multicolumn{2}{c}{[0.4, 0.6]}
&  \multicolumn{2}{c}{[0.49, 0.51]}\\
\cmidrule(lr){1-2} \cmidrule(lr){3-4} \cmidrule(lr){5-6}  \cmidrule(lr){7-8}  \cmidrule(lr){9-10} \cmidrule(lr){11-12} 
\multicolumn{2}{c}{Metrics} &  Coverage  &  Length  
&  Coverage  &  Length  &  Coverage  &  Length
&  Coverage  &  Length  &  Coverage  &  Length  \\
\midrule
\midrule
\multicolumn{2}{c}{Vanilla CQR}  
&  $89.87$\scriptsize{$\pm 1.68$}  &  $1.57$\scriptsize{$\pm 0.12$}  
&  $90.13$\scriptsize{$\pm 0.89$}  &  $1.71$\scriptsize{$\pm 0.18$} 
&  $89.87$\scriptsize{$\pm 1.06$}  &  $1.74$\scriptsize{$\pm 0.16$}  
&  $89.27$\scriptsize{$\pm 0.94$}  &  $2.07$\scriptsize{$\pm 0.55$}  
&  $89.57$\scriptsize{$\pm 0.49$}  &  $1.99$\scriptsize{$\pm 0.12$}  \\
\multicolumn{2}{c}{FCQR}  
&  $90.83$\scriptsize{$\pm 1.53$}  &  $1.19$\scriptsize{$\pm 0.19$}  
&  $90.43$\scriptsize{$\pm 1.32$}  &  $0.49$\scriptsize{$\pm 0.38$}
&  $90.23$\scriptsize{$\pm 1.13$}  &  $0.37$\scriptsize{$\pm 0.06$}  
&  $90.18$\scriptsize{$\pm 1.77$}  &  $0.20$\scriptsize{$\pm 0.05$}  
&  $89.47$\scriptsize{$\pm 0.85$}  &  $0.23$\scriptsize{$\pm 0.07$}  \\
\multirow{5}{*}{FFCQR}  &  layer0  
&  $88.92$\scriptsize{$\pm 2.78$}  &  $1.62$\scriptsize{$\pm 0.12$}  
&  $89.62$\scriptsize{$\pm 1.75$}  &  $1.67$\scriptsize{$\pm 0.07$}  
&  $91.53$\scriptsize{$\pm 0.97$}  &  $1.62$\scriptsize{$\pm 0.12$}  
&  $89.77$\scriptsize{$\pm 1.64$}  &  $1.80$\scriptsize{$\pm 0.27$}  
&  $89.82$\scriptsize{$\pm 1.07$}  &  $1.76$\scriptsize{$\pm 0.11$}  \\
{} &  layer1  
&  $88.67$\scriptsize{$\pm 2.40$}  &  $1.59$\scriptsize{$\pm 0.12$}  
&  $89.57$\scriptsize{$\pm 1.03$}  &  $1.64$\scriptsize{$\pm 0.08$}
&  $90.58$\scriptsize{$\pm 1.21$}  &  $1.57$\scriptsize{$\pm 0.13$}  
&  $89.82$\scriptsize{$\pm 1.75$}  &  $1.78$\scriptsize{$\pm 0.33$}  
&  $89.12$\scriptsize{$\pm 1.40$}  &  $1.74$\scriptsize{$\pm 0.09$}  \\
{} &  layer2  
&  $89.77$\scriptsize{$\pm 2.14$}  &  $1.58$\scriptsize{$\pm 0.12$}  
&  $89.92$\scriptsize{$\pm 1.98$}  &  $1.63$\scriptsize{$\pm 0.12$}
&  $90.53$\scriptsize{$\pm 0.43$}  &  $1.64$\scriptsize{$\pm 0.14$}  
&  $89.67$\scriptsize{$\pm 1.28$}  &  $1.89$\scriptsize{$\pm 0.38$}
&  $88.77$\scriptsize{$\pm 0.75$}  &  $1.78$\scriptsize{$\pm 0.11$}  \\
{} &  layer3  
&  $90.08$\scriptsize{$\pm 2.28$}  &  $1.58$\scriptsize{$\pm 0.12$}
&  $89.92$\scriptsize{$\pm 1.22$}  &  $1.67$\scriptsize{$\pm 0.15$}  
&  $90.33$\scriptsize{$\pm 0.86$}  &  $1.73$\scriptsize{$\pm 0.13$}  
&  $89.62$\scriptsize{$\pm 0.83$}  &  $2.03$\scriptsize{$\pm 0.54$}  
&  $89.27$\scriptsize{$\pm 0.66$}  &  $1.93$\scriptsize{$\pm 0.11$}  \\
{}  &  layer4  
&  $89.87$\scriptsize{$\pm 1.68$}  &  $1.57$\scriptsize{$\pm 0.12$}  
&  $90.13$\scriptsize{$\pm 0.89$}  &  $1.71$\scriptsize{$\pm 0.18$} 
&  $89.87$\scriptsize{$\pm 1.06$}  &  $1.74$\scriptsize{$\pm 0.16$}  
&  $89.27$\scriptsize{$\pm 0.94$}  &  $2.07$\scriptsize{$\pm 0.55$}  
&  $89.57$\scriptsize{$\pm 0.49$}  &  $1.99$\scriptsize{$\pm 0.12$}  \\
\bottomrule
\end{tabular}
}
\end{sc}
\end{small}
\end{center}
\end{table*}

\begin{table*}[h]
\caption{Coverage and Band Length at Different Confidence Levels Used By the Neural Networks in CQR methods with \emph{bike} dataset. FFCQR yields shorter band lengths compared to CQR.}
\label{tab:ffcqr-bike}
\begin{center}
\begin{small}
\begin{sc}
\resizebox{\linewidth}{!}{
\begin{tabular}{cccccccccccc}
\toprule
\multicolumn{2}{c}{Confidence Levels Uesd by NN}  &  \multicolumn{2}{c}{[0.1, 0.9]}  &  \multicolumn{2}{c}{[0.2, 0.8]} &  \multicolumn{2}{c}{[0.3, 0.7]}&  \multicolumn{2}{c}{[0.4, 0.6]}
&  \multicolumn{2}{c}{[0.49, 0.51]}\\
\cmidrule(lr){1-2} \cmidrule(lr){3-4} \cmidrule(lr){5-6}  \cmidrule(lr){7-8}  \cmidrule(lr){9-10} \cmidrule(lr){11-12} 
\multicolumn{2}{c}{Metrics} &  Coverage  &  Length  
&  Coverage  &  Length  &  Coverage  &  Length
&  Coverage  &  Length  &  Coverage  &  Length  \\
\midrule
\midrule
\multicolumn{2}{c}{Vanilla CQR}  &  $89.38$\scriptsize{$\pm 0.73$}  &
 $0.82$\scriptsize{$\pm 0.07$}  &  $89.99$\scriptsize{$\pm 0.69$}  &  $0.73$\scriptsize{$\pm 0.03$} 
&  $89.63$\scriptsize{$\pm 0.84$}  &  $0.77$\scriptsize{$\pm 0.03$}  
&  $90.25$\scriptsize{$\pm 0.62$}  &  $0.84$\scriptsize{$\pm 0.02$}  
&  $89.72$\scriptsize{$\pm 0.51$}  &  $0.96$\scriptsize{$\pm 0.08$}  \\
\multicolumn{2}{c}{FCQR}  
&  $90.25$\scriptsize{$\pm 0.67$}  &  $0.58$\scriptsize{$\pm 0.15$}  
&  $90.14$\scriptsize{$\pm 0.63$}  &  $0.65$\scriptsize{$\pm 0.13$}  
&  $89.77$\scriptsize{$\pm 0.76$}  &  $0.71$\scriptsize{$\pm 0.18$}  
&  $89.93$\scriptsize{$\pm 0.35$}  &  $0.82$\scriptsize{$\pm 0.07$}  
&  $89.98$\scriptsize{$\pm 0.97$}  &  $0.74$\scriptsize{$\pm 0.08$}  \\
\multirow{5}{*}{FFCQR}  &  layer0  
&  $89.91$\scriptsize{$\pm 0.38$}  &  $0.91$\scriptsize{$\pm 0.07$}  
&  $89.84$\scriptsize{$\pm 0.44$}  &  $0.97$\scriptsize{$\pm 0.02$}  
&  $89.42$\scriptsize{$\pm 0.83$}  &  $1.20$\scriptsize{$\pm 0.04$}  
&  $89.75$\scriptsize{$\pm 0.54$}  &  $1.64$\scriptsize{$\pm 0.13$}  
&  $89.61$\scriptsize{$\pm 0.59$}  &  $1.79$\scriptsize{$\pm 0.08$}  \\
{} &  layer1  &  
$89.57$\scriptsize{$\pm 0.33$}  &  $0.90$\scriptsize{$\pm 0.07$}  
&  $89.83$\scriptsize{$\pm 0.25$}  &$0.90$\scriptsize{$\pm 0.05$}
&  $89.44$\scriptsize{$\pm 0.42$}  &  $1.04$\scriptsize{$\pm 0.07$}  
&  $89.72$\scriptsize{$\pm 0.43$}  &  $1.25$\scriptsize{$\pm 0.05$}  
&  $89.62$\scriptsize{$\pm 0.73$}  &  $1.31$\scriptsize{$\pm 0.06$}  \\
{} &  layer2  &  $89.73$\scriptsize{$\pm 0.29$}  &  $0.87$\scriptsize{$\pm 0.07$}  &  $89.70$\scriptsize{$\pm 0.27$}  &  {$0.79$\scriptsize{$\pm 0.03$}}  
&  $89.63$\scriptsize{$\pm 0.69$}  &  $0.83$\scriptsize{$\pm 0.03$}  
&  $89.14$\scriptsize{$\pm 0.42$}  &  $0.92$\scriptsize{$\pm 0.04$}
&  $89.44$\scriptsize{$\pm 0.41$}  &  $0.98$\scriptsize{$\pm 0.04$}  \\
{} &  layer3  &  $89.49$\scriptsize{$\pm 0.34$}  &  {$0.84$\scriptsize{$\pm 0.06$}}  &  $89.62$\scriptsize{$\pm 0.48$}  &  $0.69$\scriptsize{$\pm 0.02$}  
&  $89.58$\scriptsize{$\pm 0.74$}  &  $0.69$\scriptsize{$\pm 0.02$}  
&  $89.86$\scriptsize{$\pm 0.37$}  &  $0.70$\scriptsize{$\pm 0.01$}  
&  $89.57$\scriptsize{$\pm 0.88$}  &  $0.78$\scriptsize{$\pm 0.07$}  \\
{}  &  layer4  &  $89.38$\scriptsize{$\pm 0.73$}  &  $0.82$\scriptsize{$\pm 0.07$}  &  $89.99$\scriptsize{$\pm 0.69$}  &  $0.73$\scriptsize{$\pm 0.03$}  
&  $89.63$\scriptsize{$\pm 0.84$}  &  $0.77$\scriptsize{$\pm 0.03$}  
&  $90.25$\scriptsize{$\pm 0.62$}  &  $0.84$\scriptsize{$\pm 0.02$}  
&  $89.72$\scriptsize{$\pm 0.51$}  &  $0.96$\scriptsize{$\pm 0.08$}\\
\bottomrule
\end{tabular}
}
\end{sc}
\end{small}
\end{center}
\end{table*}

\begin{table*}[h]
\caption{Coverage and Band Length at Different Confidence Levels Used By the Neural Networks in CQR methods with \emph{bio} dataset. FFCQR yields shorter band lengths compared to CQR.}
\label{tab:ffcqr-bio}
\begin{center}
\begin{small}
\begin{sc}
\resizebox{\linewidth}{!}{
\begin{tabular}{cccccccccccc}
\toprule
\multicolumn{2}{c}{Confidence Levels Uesd by NN}  &  \multicolumn{2}{c}{[0.1, 0.9]}  &  \multicolumn{2}{c}{[0.2, 0.8]} &  \multicolumn{2}{c}{[0.3, 0.7]}&  \multicolumn{2}{c}{[0.4, 0.6]}
&  \multicolumn{2}{c}{[0.49, 0.51]}\\
\cmidrule(lr){1-2} \cmidrule(lr){3-4} \cmidrule(lr){5-6}  \cmidrule(lr){7-8}  \cmidrule(lr){9-10} \cmidrule(lr){11-12} 
\multicolumn{2}{c}{Metrics} &  Coverage  &  Length  
&  Coverage  &  Length  &  Coverage  &  Length
&  Coverage  &  Length  &  Coverage  &  Length  \\
\midrule
\midrule
\multicolumn{2}{c}{Vanilla CQR}  &  $89.89$\scriptsize{$\pm 0.41$}  &
 $1.42$\scriptsize{$\pm 0.02$}  &  $89.84$\scriptsize{$\pm 0.27$}  &  $1.45$\scriptsize{$\pm 0.02$} 
&  $89.87$\scriptsize{$\pm 0.27$}  &  $1.61$\scriptsize{$\pm 0.02$}  
&  $90.07$\scriptsize{$\pm 0.31$}  &  $1.86$\scriptsize{$\pm 0.03$}  
&  $90.16$\scriptsize{$\pm 0.40$}  &  $2.00$\scriptsize{$\pm 0.03$}  \\
\multicolumn{2}{c}{FCQR}  
&  $90.18$\scriptsize{$\pm 0.35$}  &  $0.95$\scriptsize{$\pm 0.50$}  
&  $90.45$\scriptsize{$\pm 0.45$}  &  $2.09$\scriptsize{$\pm 0.41$}
&  $90.16$\scriptsize{$\pm 0.48$}  &  $1.84$\scriptsize{$\pm 0.43$}  
&  $90.25$\scriptsize{$\pm 0.46$}  &  $2.37$\scriptsize{$\pm 0.76$}  
&  $90.21$\scriptsize{$\pm 0.46$}  &  $2.02$\scriptsize{$\pm 0.34$}  \\
\multirow{5}{*}{FFCQR}  &  layer0  &  $89.74$\scriptsize{$\pm 0.32$}  &  ${1.47}$\scriptsize{$\pm 0.01$}  &  $89.98$\scriptsize{$\pm 0.22$}  &  $1.56$\scriptsize{$\pm 0.04$}  
&  $89.89$\scriptsize{$\pm 0.25$}  &  $1.73$\scriptsize{$\pm 0.04$}  
&  $89.87$\scriptsize{$\pm 0.24$}  &  $2.22$\scriptsize{$\pm 0.15$}  
&  $89.64$\scriptsize{$\pm 0.20$}  &  $2.55$\scriptsize{$\pm 0.06$}  \\
{} &  layer1  
&  $89.77$\scriptsize{$\pm 0.33$}  &  $1.45$\scriptsize{$\pm 0.01$}  
&  $89.99$\scriptsize{$\pm 0.21$}  &  $1.48$\scriptsize{$\pm 0.03$} 
&  $89.92$\scriptsize{$\pm 0.37$}  &  $1.59$\scriptsize{$\pm 0.03$}  
&  $89.92$\scriptsize{$\pm 0.21$}  &  $1.99$\scriptsize{$\pm 0.12$}  
&  $89.69$\scriptsize{$\pm 0.28$}  &  $2.21$\scriptsize{$\pm 0.04$}  \\
{} &  layer2  
&  $89.77$\scriptsize{$\pm 0.40$}  &  $1.43$\scriptsize{$\pm 0.02$}  
&  $90.01$\scriptsize{$\pm 0.23$}  &  $1.41$\scriptsize{$\pm 0.01$} 
&  $90.02$\scriptsize{$\pm 0.32$}  &  $1.49$\scriptsize{$\pm 0.03$}  
&  $90.01$\scriptsize{$\pm 0.49$}  &  $1.76$\scriptsize{$\pm 0.11$}
&  $89.79$\scriptsize{$\pm 0.35$}  &  $1.94$\scriptsize{$\pm 0.07$}  \\
{} &  layer3  &  $89.75$\scriptsize{$\pm 0.41$}  &  {$1.41  $\scriptsize{$\pm 0.02$}}  
&  $89.98$\scriptsize{$\pm 0.34$}  &  $1.38$\scriptsize{$\pm 0.02$}  
&  $89.93$\scriptsize{$\pm 0.41$}  &  $1.47$\scriptsize{$\pm 0.01$}  
&  $90.07$\scriptsize{$\pm 0.12$}  &  $1.68$\scriptsize{$\pm 0.04$}  
&  $89.97$\scriptsize{$\pm 0.34$}  &  $1.78$\scriptsize{$\pm 0.02$}  \\
{}  &  layer4  &  $89.89$\scriptsize{$\pm 0.41$}  &  ${1.42}$\scriptsize{$\pm 0.02$}  &  $89.84$\scriptsize{$\pm 0.27$}  &  $1.45$\scriptsize{$\pm 0.02$}  
&  $89.87$\scriptsize{$\pm 0.27$}  &  $1.61$\scriptsize{$\pm 0.02$}  
&  $90.07$\scriptsize{$\pm 0.31$}  &  $1.86$\scriptsize{$\pm 0.03$}  
&  $90.16$\scriptsize{$\pm 0.40$}  &  $2.00$\scriptsize{$\pm 0.03$}\\
\bottomrule
\end{tabular}
}
\end{sc}
\end{small}
\end{center}  
\end{table*}

\subsection{Group coverage}
\label{appendix: FFLCP}
\emph{Group coverage} is represented by the conditional probability \(\mathbb{P}(Y \in \mathcal{C}(X) | X)\). 
The test dataset was categorized into three groups by splitting response $Y$ based on the lower and upper tertiles, and we have reported the minimum coverage for each group. 

We present our results in two parts: (a) we present the group coverage provided by Vanilla CP, FCP, FFCP, detailed in Table~\ref{tab:group-coverage-VCP-FCP-FFCP} and (b) the group coverage provided by LCP and FFLCP, as shown in Table~\ref{tab:group-coverage-LCP-FFLCP}.

Analyzing the experimental results, we believe that the group coverage achieved through gradient-level techniques in FFCP reflects an improvement over Vanilla CP, albeit with moderate overall performance. 
We note that the group coverage of gradient-level conformal prediction is contingent upon its vanilla version. 
That is, when the vanilla version demonstrates satisfying group coverage, the gradient-level version tends to mirror this result. 
Thus, despite FFCP outperforming Vanilla CP, the overall performance is still considered average.

LCP, developed specifically to enhance group coverage, inherently achieves higher coverage. 
Experimental results further reveal that FFLCP surpasses LCP, demonstrating the superiority of our gradient-level techniques.

\begin{table*}[t]
\caption{{Comparison of Vanilla CP, FCP and FFCP in group coverage.}}
\label{tab:group-coverage-VCP-FCP-FFCP}
\begin{center}
\begin{small}
\begin{sc}
\resizebox{\linewidth}{!}{
\begin{tabular}{cccccccc}
\toprule
Method  &  \multicolumn{1}{c}{Vanilla CP}   &  \multicolumn{1}{c}{FCP} & \multicolumn{5}{c}{FFCP} \\
\cmidrule(lr){2-2} \cmidrule(lr){3-3} \cmidrule(lr){4-8} 
Dataset &  Coverage  &  Coverage  
&  layer0  &  layer1 &  layer2 &  layer3 &  layer4  \\
\midrule
\midrule
synthetic  
& $87.08$\scriptsize{$\pm 1.03$} 
& $87.92$\scriptsize{$\pm 1.08$} 
& $86.96$\scriptsize{$\pm 0.81$} 
& $86.63$\scriptsize{$\pm 0.79$}
& $85.64$\scriptsize{$\pm 1.13$}
& $\textbf{88.46}$\scriptsize{$\pm 1.44$}
& $87.08$\scriptsize{$\pm 1.03$}\\
com  
& $79.41$\scriptsize{$\pm 3.12$} 
& ${79.57}$\scriptsize{$\pm 2.96$}
& $\textbf{82.00}$\scriptsize{$\pm 3.18$} 
& ${79.41}$\scriptsize{$\pm 3.62$}
& ${78.64}$\scriptsize{$\pm 4.35$}
& ${78.65}$\scriptsize{$\pm 3.62$}
& ${79.41}$\scriptsize{$\pm 3.12$}
\\
fb1 
& $56.69$\scriptsize{$\pm 1.35$}
& $57.34$\scriptsize{$\pm 1.12$}
& $\textbf{79.20}$\scriptsize{$\pm 0.95$}
& $76.75$\scriptsize{$\pm 1.42$}
& $68.09$\scriptsize{$\pm 1.76$}
& $59.33$\scriptsize{$\pm 1.91$}
& $56.69$\scriptsize{$\pm 1.35$}
\\
fb2 
& $57.98$\scriptsize{$\pm 1.28$}
& $58.72$\scriptsize{$\pm 0.87$}
& $\textbf{76.27}$\scriptsize{$\pm 0.92$}
& ${75.64}$\scriptsize{$\pm 0.91$}
& ${70.86}$\scriptsize{$\pm 0.89$}
& ${62.43}$\scriptsize{$\pm 1.15$}
& ${57.98}$\scriptsize{$\pm 1.28$}
\\
meps19  
& $73.78$\scriptsize{$\pm 1.08$}  
& $\textbf{73.82}$\scriptsize{$\pm 0.91$}
& $70.90$\scriptsize{$\pm 2.29$}   
& $70.51$\scriptsize{$\pm 2.28$} 
& $72.09$\scriptsize{$\pm 1.25$}
& $73.53$\scriptsize{$\pm 1.00$}
& $73.78$\scriptsize{$\pm 1.08$}
\\
meps20  
& $72.21$\scriptsize{$\pm 1.47$}  
& $\textbf{72.33}$\scriptsize{$\pm 1.46$}  
& $70.42$\scriptsize{$\pm 0.88$} 
& $70.13$\scriptsize{$\pm 1.42$}
& $69.51$\scriptsize{$\pm 0.79$}
& $71.17$\scriptsize{$\pm 2.01$}
& $72.21$\scriptsize{$\pm 1.47$}
\\
meps21 
& $71.38$\scriptsize{$\pm 0.20$} 
& $\textbf{72.02}$\scriptsize{$\pm 0.70$} 
& $69.40$\scriptsize{$\pm 1.61$} 
& $69.83$\scriptsize{$\pm 1.44$} 
& $69.81$\scriptsize{$\pm 1.68$}
& $70.85$\scriptsize{$\pm 0.82$}
& $71.39$\scriptsize{$\pm 0.20$}
\\
star 
& $\textbf{83.45}$\scriptsize{$\pm 3.09$} 
& $83.17$\scriptsize{$\pm 3.47$}
& $82.89$\scriptsize{$\pm 1.51$} 
& $81.22$\scriptsize{$\pm 2.55$}
& $81.22$\scriptsize{$\pm 3.60$}
& $83.03$\scriptsize{$\pm 2.07$}
& $\textbf{83.45}$\scriptsize{$\pm 3.09$}
\\
bio &  $81.00$\scriptsize{$\pm 0.61$}&   {$ {84.45}$\scriptsize{$\pm 0.88$}}&  $87.31$\scriptsize{$\pm 0.27$}&  ${87.27}$\scriptsize{$\pm 0.46$}
& $\textbf{88.31}$\scriptsize{$\pm 0.72$}
& ${84.20}$\scriptsize{$\pm 0.70$}
& ${81.00}$\scriptsize{$\pm 0.61$}
\\
blog
& $58.32$\scriptsize{$\pm 0.90$}
& $60.43$\scriptsize{$\pm 1.46$}
& $\textbf{65.21}$\scriptsize{$\pm 0.58$}
& $59.03$\scriptsize{$\pm 1.03$}
& $54.55$\scriptsize{$\pm 0.77$}
& $55.76$\scriptsize{$\pm 1.26$}
& $58.32$\scriptsize{$\pm 0.90$}
\\
bike 
& $77.55$\scriptsize{$\pm 1.40$}
& ${86.25}$\scriptsize{$\pm 0.87$}
& $\textbf{95.36}$\scriptsize{$\pm 1.32$}
& $94.23$\scriptsize{$\pm 1.40$}
& ${95.06}$\scriptsize{$\pm 1.06$}
& ${84.65}$\scriptsize{$\pm 1.85$}
& ${77.55}$\scriptsize{$\pm 1.40$}
\\
\bottomrule
\end{tabular}
}
\end{sc}
\end{small}
\end{center}
\end{table*}

\begin{algorithm*}[t]
\caption{Fast Feature Localized Conformal Prediction (FFLCP)} 
\label{alg: fflcp} 
\begin{algorithmic}[1] 
\REQUIRE  Confidence level $\alpha$, dataset $\cD = \{ (\Xs, \Ys)\}_{i \in \cI}$, tesing point $X^\prime$, localizer $D(\X,\Y)$

\STATE{Randomly split the dataset $\mathcal{D}$ into a training fold $\cD_{\text{tra}} \triangleq \{(\X_i, \Y_i)\}_{i \in \mathcal{I}_{\text{tra}}}$ and a calibration fold $\cD_{\text{cal}} \triangleq \{(\X_i, \Y_i)\}_{i \in \mathcal{I}_{\text{cal}}}$ ;}

\STATE{Train a base neural network with training fold $f(\cdot) = g\circ h(\cdot)$ with training fold $\cD_{\text{tra}}$;}

\STATE{For each $i \in \mathcal{I}_{\text{cal}}$, calculate the non-conformity score 
$\tilde{R}_i = |\Y_i -  f(\X_i)|/\| \grad {g}(\hat{v}_i)\|,$ where $\grad {g}(\hat{v}_i)$ denotes the gradient of ${g}(\cdot)$ on the feature $\hat{v}_i \triangleq h(\X_i)$, namely $\grad {g}(\hat{v}_i) = \frac{\mathrm{d} {g}\circ {h}(\X_i)}{\mathrm{d} {h}(\X_i)} $;
}

\STATE{Calculate the distance $D_i \triangleq D(\X^\prime, \X_i)$, $d_{i}^D := \frac{D_{i}}{\sum_{i \in \mathcal{I}_{\text{cal}}} D_{i}}$ 
 and $(1-\alpha)$-th quantile $Q_{1-\alpha}$ of the distribution  $ \sum_{i \in \mathcal{I}_{\text{cal}}} d_i^D \delta_{\tilde{R}_i} + {\delta}_{\infty}$;}

\ENSURE $\cC_{1-\alpha}^{\text{fflcp}}(\X^\prime)= \left[ f(\X^\prime) - \|\nabla{g}(\hat{v}^\prime)\| Q_{1-\alpha}, f(\X^\prime) + \|\nabla {g}(\hat{v}^\prime)\| Q_{1-\alpha}\right]$, where $\hat{v}^\prime = h(\X^\prime)$.

\end{algorithmic} 
\end{algorithm*}

\subsection{FFRAPS}
\label{appendix: FFRAPS}
In this section, we show how to deploy gradient-level techniques in FFCP in classification problems. The basic ideas follow Algorithm~\ref{alg:ffraps}.

Comparing to the experimental part of RAPS, our core adjustments are as follows:

(a) During the calibration process, for the model's output of sorted scores $s$, we divide each element by the magnitude of its corresponding gradient: $s + \delta \cdot \|\nabla g(v)\|$.
Here, $\delta$ is an adjustable hyper-parameter that can be tuned to optimize the performance of the model based on the specific characteristics of the data and the problem at hand.

(b) In the stage of calculating the returned set, we multiply the generalized inverse quantile $\tau$ by the magnitude of the gradient of the corresponding test data: $s^\prime + \delta \cdot \|\nabla g(v^\prime)\|$

We summarize the experiment results in Table~\ref{tab1:classification}, where we adhere to the statistical methodologies of RAPS as described in~\cite{angelopoulos2020uncertainty}.
\begin{algorithm*}[t]
\caption{Fast Feature Regularized Adaptive Prediction Sets (FFRAPS)}
\label{alg:ffraps}
\begin{algorithmic}[1] %

\REQUIRE 
Confidence level $\alpha$, dataset $\cD = \{ (\Xs, \Ys)\}_{i \in \cI}$, tesing point $X^\prime$, and ground-truth label $y \in \{0,1,...,K\}^n$ for $\X \in \cD$ and $X^\prime$; regularization hyperparameters $k_{reg}$, $\delta$ and $\lambda$;
\STATE{Randomly split the dataset $\mathcal{D}$ into a training fold $\cD_{\text{tra}} \triangleq \{(\X_i, \Y_i)\}_{i \in \mathcal{I}_{\text{tra}}}$ and a calibration fold $\cD_{\text{cal}} \triangleq \{(\X_i, \Y_i)\}_{i \in \mathcal{I}_{\text{cal}}}$ ;}

\STATE{Train a base neural network with training fold $f(\cdot) = g\circ h(\cdot)$ with training fold $\cD_{\text{tra}}$;}

\STATE{For each $i \in \mathcal{I}_{\text{cal}}$, $L_i \gets j \text{ such that } I_{i,j} = y_i$, where $I$ represents the associated permutation of index. Calculate generalized inverse quantile conformity score 
$E_i = \Sigma_{j=1}^{L_i}s_{i,j}+\| \grad {g}(\hat{v}_i)\| \cdot \delta +\lambda(L_i-k_{reg})^+$
, where $\grad {g}(\hat{v}_i)$ denotes the gradient of ${g}(\cdot)$ on the feature $\hat{v}_i \triangleq h(\X_i)$, namely $\grad {g}(\hat{v}_i) = \frac{\mathrm{d} {g}\circ {h}(\X_i)}{\mathrm{d} {h}(\X_i)} $, where $s\triangleq\text{sort}f(\X)$ represents the sorted scores. Calculate $\hat{\tau}_{ccal} \gets  \lceil (1-\alpha)(1+n) \rceil$ largest value in $\{E_i\}_{i=1}^n$}

\STATE  Calculate $L \gets |\;\{ j \in \mathcal{Y} \; : \;  \Sigma_{i=1}^js^\prime_i + \| \grad {g}(\hat{v}^\prime_i)\| \cdot \delta + \lambda(j-k_{reg})^+ \leq \hat{\tau}_{ccal} * \}\;|+1$, where $\hat{v}^\prime = h(\X^\prime)$ and $s^\prime = \text{sort}f(\X^\prime)$;

\ENSURE $\cC_{1-\alpha}^{\text{FFRAPS}}(\X^\prime) = \big\{I_1, ... I_L\big\}$
\end{algorithmic}
\end{algorithm*}

\subsection{Additional Experiment Results}
\label{appendix:Additional-Experiment-Results}

This section provides more experiment results. Additional visual results for the segmentation problem are also presented in Figure~\ref{fig:add-seg}.

\begin{figure}[t]
    \centering
    \subfigure{\includegraphics[height=0.7\linewidth]{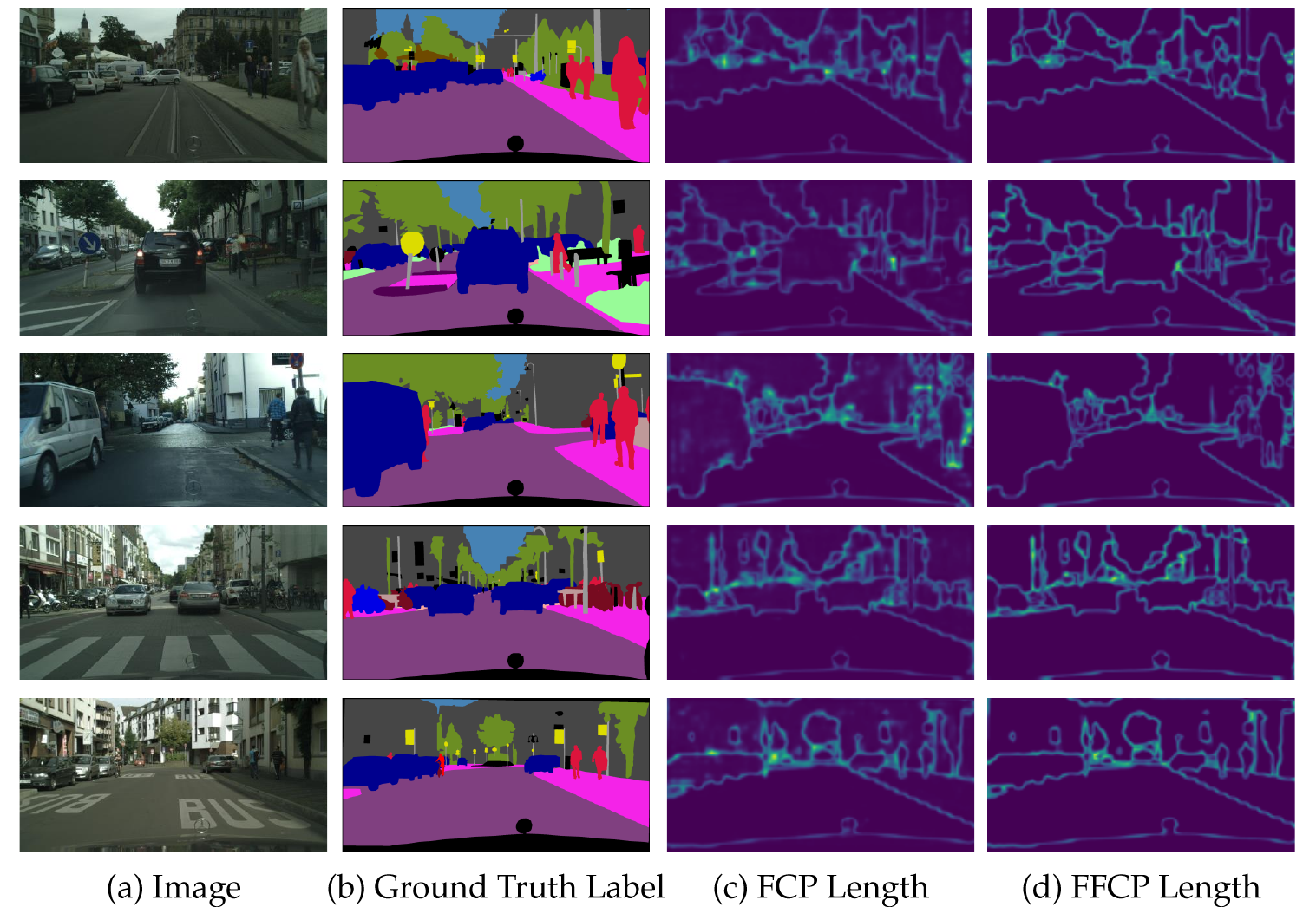}}
    \caption{Additional visualization results in segmentation task.}
    \label{fig:add-seg}
\end{figure}

\end{document}